\pdfoutput=1
\documentclass{article}

\usepackage{microtype}
\usepackage{graphicx}
\usepackage{subfigure}
\usepackage{booktabs} 

\usepackage{amsmath}
\usepackage{amssymb}
\usepackage{mathtools}
\usepackage{amsthm}
\usepackage{booktabs, graphicx, color, verbatim, wrapfig}
\usepackage{exscale,cmmib57}
\usepackage{algorithm}
\usepackage{algorithmic}
\usepackage{rotating}
\usepackage{tabulary}
\usepackage{multirow}
\usepackage{bm}

\usepackage{hyperref}
\usepackage{sepnum}
\usepackage{color}

\newtheorem{thm}{Theorem}[section]
\newtheorem{prop}{Proposition}[section]
\newtheorem{lem}{Lemma}[section]

\newtheorem{assum}{Assumption}[section]

\usepackage[accepted]{icml2023}

\usepackage[capitalize,noabbrev]{cleveref}

\usepackage[textsize=tiny]{todonotes}

\icmltitlerunning{Existence and Estimation of Critical Batch Size for Training GANs with TTUR}

\begin{document}

\twocolumn[
\icmltitle{Existence and Estimation of Critical Batch Size for Training Generative Adversarial Networks with Two Time-Scale Update Rule}

\icmlsetsymbol{equal}{*}

\begin{icmlauthorlist}
\icmlauthor{Naoki Sato}{equal,yyy}
\icmlauthor{Hideaki Iiduka}{equal,yyy}
\end{icmlauthorlist}

\icmlaffiliation{yyy}{Department of Computer Science, Meiji University, Japan}

\icmlcorrespondingauthor{Naoki Sato}{ce235017@meiji.ac.jp}
\icmlcorrespondingauthor{Hideaki Iiduka}{iiduka@cs.meiji.ac.jp}

\icmlkeywords{adaptive method, batch size, critical batch size, GANs, non-convex optimization}

\vskip 0.3in
]

\printAffiliationsAndNotice{\icmlEqualContribution} 

\begin{abstract}
Previous results have shown that a two time-scale update rule (TTUR) using different learning rates, such as different constant rates or different decaying rates, is useful for training generative adversarial networks (GANs) in theory and in practice. Moreover, not only the learning rate but also the batch size is important for training GANs with TTURs and they both affect the number of steps needed for training. This paper studies the relationship between batch size and the number of steps needed for training GANs with TTURs based on constant learning rates. We theoretically show that, for a TTUR with constant learning rates, the number of steps needed to find stationary points of the loss functions of both the discriminator and generator decreases as the batch size increases and that there exists a critical batch size minimizing the stochastic first-order oracle (SFO) complexity. Then, we use the Fr\'echet inception distance (FID) as the performance measure for training and provide numerical results indicating that the number of steps needed to achieve a low FID score decreases as the batch size increases and that the SFO complexity increases once the batch size exceeds the measured critical batch size. Moreover, we show that measured critical batch sizes are close to the sizes estimated from our theoretical results.
\end{abstract}

\section{Introduction}
\label{sec:1}
\subsection{Background}
Generative adversarial networks (GANs) have attracted attention (see, e.g., \citep{kiran2019,jordon2018knockoffgan, NEURIPS2020_641d77dd,zhu2021}) with the development of real-world applications \citep{ar2017,NIPS2017_892c3b1c,brock2018large,NEURIPS2019_2e2c080d}. The generator network in a GAN constructs synthetic data from random variables, and the discriminator network separates the synthetic data from the real-world data. 

Many optimizers have been presented for training GANs (see, e.g., \citep{NIPS2014_5ca3e9b1,NIPS2017_7e0a0209,Heusel2017,NEURIPS2019_58a2fc6e,NEURIPS2020_641d77dd,sauer2021}). In this paper, we focus on a two time scale update rule (TTUR) \citep{Heusel2017} for finding a pair of stationary points of the loss functions of the discriminator and generator. TTUR uses sequences generated by each of the generator and the discriminator. 

For example, let us consider a TTUR based on stochastic gradient descent (SGD) \citep{feh2020,sca2020,chen2020}; let $\bm{\theta}_n$ be the point generated by the generator at iteration $n$ and $\bm{w}_n$ be the point generated by the discriminator at iteration $n$. The loss function of the generator $L_G(\cdot, \bm{w}_n)$ is minimized by using SGD with a learning rate $\alpha_n^G$ and a mini-batch stochastic gradient at $\bm{\theta}_n$ with a batch of size $b$. The loss function of the discriminator $L_D(\bm{\theta}_n, \cdot)$ is minimized with a learning rate $\alpha_n^D$ and a mini-batch stochastic gradient at $\bm{w}_n$ with the same batch size $b$. If $\alpha_n^G$ and $\alpha_n^D$ are decaying learning rates, then TTUR based on SGD converges almost surely to a pair of stationary points of $L_G$ and  $L_D$ \citep[Theorem 1]{Heusel2017}. 
 
TTURs based on adaptive methods can be defined by replacing SGD with adaptive methods for training deep neural networks. For example, TTUR based on adaptive moment estimation (Adam) \citep{adam} with decaying learning rates $\alpha_n^G$ and $\alpha_n^D$ converges almost surely to a pair of stationary points of the loss functions of the generator and the discriminator \citep[Theorem 2]{Heusel2017}. It was shown numerically that TTUR based on AdaBelief (short for adapting step sizes by the belief in observed gradients) \citep{ada} with constant learning rates $\alpha^G$ and $\alpha^D$ has good scores in terms of the Fr\'echet inception distance (FID) \citep{Heusel2017}, which is a performance measure of optimizers for training GANs. This implies that TTUR based on AdaBelief is a powerful way to train GANs. 

\subsection{Motivation}
The above subsection described that TTUR with decaying or constant learning rates can be used in both theory and practice to train GANs. In particular, the numerical results in \citep{Heusel2017,ada} show that TTUR performs well with constant learning rates. Hence, we will set constant learning rates $\alpha^G$ and $\alpha^D$ for the generator and the discriminator. 

Meanwhile, the batch size also affects the performance of TTUR. Previous numerical evaluations \citep[Figure 5]{Heusel2017} have indicated that increasing the batch size used in TTUR tends to decrease the FID. This implies that large batch sizes are desirable for training GANs. The motivation behind this work is thus to identify the theoretical relationship between the performance of TTUR with constant learning rates and the batch size. In so doing, we may bridge the gap between theory and practice in regard to the performance of TTUR based on the batch size.

We will use the number of steps $N$ needed to train the GAN as the performance metric and examine the relationship between the batch size $b$ and number of steps $N$. Motivated by the previously reported results \citep{Heusel2017}, we tried to find a pair of stationary points of the loss functions of the discriminator and generator that satisfy the variational inequalities (\ref{prob:1}) of the gradients of both loss functions. 

Previous results \cite{shallue2019,zhang2019,iiduka2022,priya2017,hoffer2017,you2017} on training deep neural networks in practical tasks have shown that, for each deep learning optimizer, the number of training steps is halved by each doubling of the batch size and that diminishing returns exist beyond a critical batch size. On the other hand, we are interested in verifying whether a critical batch size for GANs exists in theory and in practice.

\subsection{Contribution}
\subsubsection{Theoretical result on TTUR with small constant learning rates and large batch sizes}
The theoretical contribution of this paper is to show that, for TTURs with constant learning rates, the number of steps $N$ needed to find a pair of stationary points of the loss functions of the discriminator and the generator decreases as the batch size $b$ increases (see also (\ref{NG_ND}) for the explicit forms of $N$). To show this, we need to clarify that TTUR with constant learning rates can approximate a pair $(\bm{\theta}^\star, \bm{w}^\star)$ of stationary points of the loss function $L_D (\bm{\theta}^\star, \cdot)$ of the discriminator and the loss function $L_G(\cdot,\bm{w}^\star)$ of the generator. 

We define the inner product of $\bm{x}, \bm{y} \in \mathbb{R}^d$ by $\langle \bm{x}, \bm{y} \rangle := \bm{x}^\top \bm{y}$ and the norm of $\bm{x} \in \mathbb{R}^d$ by $\|\bm{x}\| := \sqrt{\langle \bm{x}, \bm{x} \rangle}$. Let $\nabla_{\bm{\theta}} L_G (\cdot, \bm{w})$ be the gradient of $L_G (\cdot, \bm{w})$ ($\bm{w} \in \mathbb{R}^W$) and $\nabla_{\bm{w}} L_D (\bm{\theta},\cdot)$ be the gradient of $L_D (\bm{\theta},\cdot)$ ($\bm{\theta} \in \mathbb{R}^\Theta$). A pair $(\bm{\theta}^\star, \bm{w}^\star)$ of stationary points of $L_D (\bm{\theta}^\star, \cdot)$ and $L_G(\cdot,\bm{w}^\star)$ is such that
\begin{align*} 
\|\nabla_{\bm{\theta}} L_G (\bm{\theta}^\star, \bm{w}^\star)\| = 0 
\text{ and } 
\|\nabla_{\bm{w}} L_D (\bm{\theta}^\star, \bm{w}^\star)\| = 0, 
\end{align*}
which are equivalent to the following variational inequalities (see Appendix \ref{a10}): for all $\bm{\theta} \in \mathbb{R}^\Theta$ and all $\bm{w} \in \mathbb{R}^W$, 
\begin{align}\label{prob:1}
\begin{split}
&\langle \bm{\theta}^\star - \bm{\theta}, 
\nabla_{\bm{\theta}} L_G (\bm{\theta}^\star, \bm{w}^\star) \rangle \leq 0
\text{ and } \\
&\langle \bm{w}^\star - \bm{w}, \nabla_{\bm{w}} L_D (\bm{\theta}^\star, \bm{w}^\star) \rangle \leq 0. 
\end{split}
\end{align} 

Let us examine TTURs based on adaptive methods (see Algorithm \ref{algo:1} and Table \ref{table:ex} in Appendix \ref{a2}) such as Adam \citep{adam}, AdaBelief \citep{ada}, and RMSProp \citep{rmsprop}. Let $((\bm{\theta}_n, \bm{w}_n))_{n\in\mathbb{N}} \subset \mathbb{R}^\Theta \times \mathbb{R}^W$ be the sequence generated by TTUR with constant learning rates $\alpha^G$ and $\alpha^D$. We will show that, under certain assumptions, for all $\bm{\theta} \in \mathbb{R}^\Theta$ and all $\bm{w} \in \mathbb{R}^W$,
\begin{align}\label{main}
\begin{split}
\quad
&\frac{1}{N} \sum_{n=1}^N \mathbb{E} \left[ \langle \bm{\theta}_n - \bm{\theta}, \nabla_{\bm{\theta}}
L_G (\bm{\theta}_n, \bm{w}_n) \rangle \right]\\
&\quad\leq 
\underbrace{\frac{\Theta \mathrm{Dist}(\bm{\theta}) H^G}{2 \alpha^G \tilde{\beta_1^G}}}_{A_G}
\frac{1}{N}
+
\underbrace{\frac{\sigma_G^2 \alpha^G}{2 \tilde{\beta_1^G} \tilde{\gamma}^{G^2} h_{0,*}^G}}_{B_G} 
\frac{1}{b} \\
&\quad\quad+
\underbrace{\frac{M_G^2 \alpha^G}{2 \tilde{\beta_1^G} \tilde{\gamma}^{G^2} h_{0,*}^G}
+
\frac{\beta_1^G}{\tilde{\beta_1^G}}
\sqrt{\Theta \mathrm{Dist}(\bm{\theta})
( \sigma_G^2 + M_G^2 )}}_{C_G},\\
&\frac{1}{N} \sum_{n=1}^N \mathbb{E} \left[ \langle \bm{w}_n - \bm{w}, \nabla_{\bm{w}}
L_D (\bm{\theta}_n, \bm{w}_n) \rangle \right]\\
&\quad \leq
\underbrace{\frac{W \mathrm{Dist}(\bm{w}) H^D}{2 \alpha^D \tilde{\beta_1^D}}}_{A_D}
\frac{1}{N}
+
\underbrace{\frac{\sigma_D^2 \alpha^D}{2 \tilde{\beta_1^D} \tilde{\gamma}^{D^2} h_{0,*}^D}}_{B_D} 
\frac{1}{b} \\
&\quad\quad+
\underbrace{\frac{M_D^2 \alpha^D}{2 \tilde{\beta_1^D} \tilde{\gamma}^{D^2} h_{0,*}^D} 
+
\frac{\beta_1^D}{\tilde{\beta_1^D}}
\sqrt{W \mathrm{Dist}(\bm{w})
( \sigma_D^2 + M_D^2)}}_{C_D},
\end{split}
\end{align}
where $\sigma_G^2, \sigma_D^2 \geq 0$, $M_G, M_D, \mathrm{Dist}(\bm{\theta}), \mathrm{Dist}(\bm{w}) > 0$, $\beta_1^G, \beta_1^D, \gamma^G, \gamma^D \in [0,1)$, $\tilde{\beta_1^G} := 1 - \beta_1^G$, $\tilde{\gamma}^G := 1 - \gamma^G$, $\tilde{\beta_1^D} := 1 - \beta_1^D$, $\tilde{\gamma}^D := 1 - \gamma^D$, $h_{0,*}^G, h_{0,*}^D > 0$, $H^G := \max_{i\in [\Theta]} H_i^G$, and $H^D := \max_{j\in [W]} H_j^D$ (see Theorem \ref{thm:1} and Section \ref{subsec:3.1} for the definitions of the parameters). This result implies that using small constant learning rates $\alpha^G$ and $\alpha^D$ makes $B_G$, $C_G$, $B_D$, and $C_D$ small and that using a large batch size $b$ makes $B_G/b$ and $B_D/b$ small. Meanwhile, using small constant learning rates $\alpha^G$ and $\alpha^D$ makes $A_G$ and $A_D$ large. Hence, we need to use a large number of steps $N$ to make $A_G$ and $A_D$ small when small constant learning rates are used. Therefore, small constant learning rates, a large batch size, and a large number of steps are useful for training GANs. 

\subsubsection{Relationship between batch size and the number of steps needed for an $\epsilon$--approximation of TTUR}
\label{subsec:1.3.2}
Let us suppose that the generator and the discriminator run in at most $N_G$ and $N_D$ steps for a certain batch size $b$,
\begin{align}\label{NG_ND}
\begin{split}
&N_G (b) := \frac{A_G b}{(\epsilon_G^2 - C_G)b -B_G} \text{ and } \\
&N_D (b) := \frac{A_D b}{(\epsilon_D^2 - C_D)b -B_D},
\end{split}
\end{align}
where $\epsilon_G, \epsilon_D > 0$. Then, we can show that TTUR is an $\epsilon$--approximation in the sense that 
\begin{align*}
&\frac{1}{N_G} \sum_{n=1}^{N_G} \mathbb{E} \left[ \langle \bm{\theta}_n - \bm{\theta}, \nabla_{\bm{\theta}}
L_G (\bm{\theta}_n, \bm{w}_n) \rangle \right]
\leq \epsilon_G^2 
\text{ and } \\
&\frac{1}{N_D} \sum_{n=1}^{N_D} \mathbb{E} \left[ \langle \bm{w}_n - \bm{w}, \nabla_{\bm{w}}
L_D (\bm{\theta}_n, \bm{w}_n) \rangle \right]
\leq \epsilon_D^2.
\end{align*}
Moreover, $N_G$ and $N_D$ are monotone decreasing and convex functions of the batch size (Theorem \ref{thm:2}). This result implies that large batch sizes are desirable in the sense of minimizing the number of steps for training GANs. A particularly interesting concern is how large the batch size $b$ should be. 

\subsubsection{Existence of a critical batch size minimizing SFO complexity}
\label{subsec:1.3.3}
Here, we consider stochastic first-order oracle (SFO) complexities \cite{iiduka2022} defined by $N_G(b) b$ and $N_D(b) b$. We show that $N_G(b) b$ and $N_D(b) b$ are convex functions of $b$ and that there exist global minimizers $b_G^\star$ and $b_D^\star$ of $N_G(b) b$ and $N_D(b) b$ (Theorem \ref{thm:3}). Accordingly, $b_G^\star$ and $b_D^\star$ are given by 
\begin{align}\label{cri}
b_G^\star := \frac{2 B_G}{\epsilon_G^2 - C_G} \text{ and }
b_D^\star := \frac{2 B_D}{\epsilon_D^2 - C_D} \text{.}
\end{align}
It would be desirable to use critical batch sizes $b_G^\star$ and $b_D^\star$ as it minimizes the SFO complexity, which is the computation cost of the stochastic gradient. Furthermore, we show that lower bounds for $b_G^\star$ and $b_D^\star$ can be estimated from some hyperparameters, the total number of datasets, and the number of dimensions of the model (Proposition \ref{prop:4}).

\subsubsection{Numerical results supporting our theoretical results: Estimation of critical batch size minimizing SFO complexity}
We use FID as a performance measure for training a deep convolutional GAN (DCGAN) \citep{radford2016unsupervised} on the LSUN-Bedroom dataset \citep{lsun}, a Wasserstein GAN with Gradient Penalty (WGAN-GP) \citep{NIPS2017_892c3b1c} on the CelebA dataset \citep{celeba}, and a BigGAN \citep{brock2018large} on the ImageNet dataset \citep{imagenet}. We numerically show that increasing the batch size decreases the number of steps needed to achieve a low FID score and that there exist critical batch sizes minimizing the SFO complexities. The numerical results match our theoretical results in Sections \ref{subsec:1.3.2} and \ref{subsec:1.3.3}. We are also interested in estimating appropriate batch sizes before implementing TTURs. Hence, we estimate batch sizes using $b_G^\star$ and $b_D^\star$ in (\ref{cri}) and compare them with ones measured in numerical experiments. We find that the estimated sizes are close to the measured ones (see Section \ref{subsec:4.4}).

\section{Mathematical Preliminaries}
\label{sec:2}
\subsection{Assumptions}
The notation used in this paper is summarized in Table \ref{notation}.

\begin{table*}[htbp]
\small
\begin{center}	
\caption{Notation List (The parameters and functions of the discriminator are defined by replacing $G$, $\mathcal{S}_n$, $\bm{\theta}$, and $\bm{w}$ in the generator with $D$, $\mathcal{R}_n$, $\bm{w}$, and $\bm{\theta}$)}\label{notation}
\begin{tabular}{c||l} \hline
Notation & Description\\ \hline
$\mathbb{N}$ & 
The set of all nonnegative integers \\

$[N]$ & 
$[N] := \{1,2,\ldots,N\}$ ($N \in \mathbb{N}\backslash \{ 0 \}$)\\

$|A|$ &
The number of elements of a set $A$ \\

$\mathbb{R}^d$ & 
A $d$-dimensional Euclidean space with inner product $\langle \cdot, \cdot \rangle$, which induces the norm $\| \cdot \|$ \\

$\mathbb{R}_+^d$ & 
$\mathbb{R}_+^d := \{ \bm{x} \in \mathbb{R}^d \colon x_i \geq 0 \text{ } (i\in [d]) \}$ \\

$\mathbb{R}_{++}^d$ & 
$\mathbb{R}_{++}^d := \{ \bm{x} \in \mathbb{R}^d \colon x_i > 0 \text{ } (i\in [d]) \}$ \\

$\mathbb{S}_{++}^d$ &
The set of $d\times d$ symmetric positive-definite matrices\\

$\mathbb{D}^d$ &
The set of $d\times d$ diagonal matrices, i.e., $\mathbb{D}^d = \{ M \in \mathbb{R}^{d \times d} \colon M = \mathsf{diag}(x_i), \text{ } x_i \in \mathbb{R} \text{ } (i\in [d]) \}$\\

$\mathbb{E}_\xi [X]$ & 
The expectation with respect to $\xi$ of a random variable $X$ \\

$\mathcal{S}$ &
A set of synthetic samples $\bm{z}^{(i)}$ \\

$\mathcal{R}$ &
A set of real-world samples $\bm{x}^{(i)}$ \\

$\mathcal{S}_n$ &
Mini-batch of $b$ synthetic samples $\bm{z}^{(i)}$ at time $n$ \\

$\mathcal{R}_n$ &
Mini-batch of $b$ real world samples $\bm{x}^{(i)}$ at time $n$ \\

\hline
$L_G^{(i)}(\cdot, \bm{w})$ &
A loss function of the generator for $\bm{w} \in \mathbb{R}^{W}$ and $\bm{z}^{(i)}$ \\

$L_G (\cdot, \bm{w})$ &
The total loss function of the generator for $\bm{w} \in \mathbb{R}^W$, i.e., $L_G(\cdot,\bm{w}) := |\mathcal{S}|^{-1} \sum_{i\in \mathcal{S}} L_G^{(i)}(\cdot, \bm{w})$ \\

$\xi^G$ &
A random variable supported on $\Xi^G$ that does not depend on $\bm{w} \in \mathbb{R}^W$ and $\bm{\theta} \in \mathbb{R}^\Theta$\\

$\xi_n^G$ &
$\xi_0^G, \xi_1^G,\ldots$ are independent samples and $\xi_n^G$ is independent of $(\bm{\theta}_k)_{k=0}^n \subset \mathbb{R}^\Theta$ and $\bm{w} \in \mathbb{R}^W$\\

$\xi_{n,i}^G$ &
A random variable generated from the $i$-th sampling at time $n$\\

$\xi_{[n]}^G$ &
The history of process $\xi_0^G, \xi_1^G,\ldots$ to time step $n$, i.e., $\xi_{[n]}^G := (\xi_0^G, \xi_1^G,\ldots, \xi_n^G)$\\

$\mathsf{G}_{\xi^G}(\bm{\theta})$ &
The stochastic gradient of $L_G (\cdot,\bm{w})$ at $\bm{\theta} \in \mathbb{R}^\Theta$\\

$\nabla L_{G,\mathcal{S}_n} (\bm{\theta}_n)$ &
The mini-batch stochastic gradient of $L_G(\bm{\theta}_n, \bm{w}_n)$ for $\mathcal{S}_n$, i.e., $\nabla L_{G,\mathcal{S}_n} (\bm{\theta}_n) := b^{-1} \sum_{i\in [b]} \mathsf{G}_{\xi_{n,i}^G}(\bm{\theta}_n)$
\\
\hline
\end{tabular}
\end{center}
\end{table*}

We assume the following standard conditions:

\begin{assum}\label{assum:0}
\text{ } 

{\em (S1)} $L_G^{(i)}(\cdot, \bm{w}) \colon \mathbb{R}^\Theta \to \mathbb{R}$ and $L_D^{(i)}(\bm{\theta}, \cdot) \colon \mathbb{R}^W \to \mathbb{R}$ are continuously differentiable.

{\em (S2)} Let $((\bm{\theta}_n, \bm{w}_n))_{n\in \mathbb{N}} \subset \mathbb{R}^\Theta \times \mathbb{R}^W$ be the sequence generated by an optimizer. 

{\em (i)} For each iteration $n$, 
\begin{align*}
&\mathbb{E}_{\xi_n^G}\left[ \mathsf{G}_{\xi_n^G}(\bm{\theta}_n) \right]
= \nabla_{\bm{\theta}} L_G (\bm{\theta}_n, \bm{w}_n) \text{ and } \\
&\mathbb{E}_{\xi_n^D}\left[ \mathsf{G}_{\xi_n^D}(\bm{w}_n) \right]
= \nabla_{\bm{w}} L_D (\bm{\theta}_n, \bm{w}_n).
\end{align*} 
{\em (ii)} There exist nonnegative constants $\sigma_G^2$ and $\sigma_D^2$ such that
\begin{align*}
&\mathbb{E}_{\xi_n^G}\left[ \| \mathsf{G}_{\xi_n^G}(\bm{\theta}_n) 
- \nabla_{\bm{\theta}} L_G (\bm{\theta}_n, \bm{w}_n) \|^2 \right]
\leq \sigma_G^2 \text{ and } \\
&\mathbb{E}_{\xi_n^D}\left[ \| \mathsf{G}_{\xi_n^D}(\bm{w}_n) 
- \nabla_{\bm{w}} L_D (\bm{\theta}_n, \bm{w}_n) \|^2 \right]
\leq \sigma_D^2.
\end{align*}

{\em (S3)} For each iteration $n$, the optimizer samples mini-batches $\mathcal{S}_n \subset \mathcal{S}$ and $\mathcal{R}_n \subset \mathcal{R}$ and estimates the full gradients $\nabla L_G$ and $\nabla L_D$ as 
\begin{align*}
\nabla L_{G,\mathcal{S}_n} (\bm{\theta}_n)
&:= \frac{1}{b} \sum_{i\in [b]} \mathsf{G}_{\xi_{n,i}^G}(\bm{\theta}_n)\\
&= \frac{1}{b} \sum_{\{i\colon\bm{z}^{(i)} \in \mathcal{S}_n\}} \nabla_{\bm{\theta}} L_{G}^{(i)}(\bm{\theta}_n, \bm{w}_n) \text{ and }\\
\nabla L_{D,\mathcal{R}_n} (\bm{w}_n) 
&:= \frac{1}{b} \sum_{i\in [b]} \mathsf{G}_{\xi_{n,i}^D}(\bm{w}_n)\\
&= \frac{1}{b} \sum_{\{i\colon\bm{x}^{(i)} \in \mathcal{R}_n\}} \nabla_{\bm{w}} L_{D}^{(i)}(\bm{\theta}_n, \bm{w}_n).
\end{align*}
\end{assum}

\subsection{Adaptive method}
We will consider the following TTUR-type optimizer, described by Algorithm \ref{algo:1}, for solving Problem (\ref{prob:1}).

\begin{algorithm} 
\caption{Adaptive Method for Solving Problem (\ref{prob:1})} 
\label{algo:1} 
\begin{algorithmic}[1] 
\REQUIRE
$(\alpha_n^G)_{n\in\mathbb{N}}, (\alpha_n^D)_{n\in\mathbb{N}} \subset \mathbb{R}_{++}$, $\beta_1^G, \beta_1^D \in [0,1)$, $\gamma^G, \gamma^D \in [0,1)$
\STATE
$n \gets 0$, $(\bm{\theta}_0, \bm{w}_0) \in \mathbb{R}^\Theta \times \mathbb{R}^W$, 
$\bm{m}_{-1}^G = \bm{0} \in \mathbb{R}^\Theta$, 
$\bm{m}_{-1}^D = \bm{0} \in \mathbb{R}^W$
\LOOP
\LOOP 
\STATE 
$\bm{m}_n^G := \beta_1^G \bm{m}_{n-1}^G + (1 - \beta_1^G) \nabla L_{G,\mathcal{S}_n}(\bm{\theta}_n)$
\STATE
$\hat{\bm{m}}_n^G := (1 - \gamma^{{G}^{n+1}})^{-1} \bm{m}_n^G$
\STATE
$\mathsf{H}_n^G \in \mathbb{S}_{++}^\Theta \cap \mathbb{D}^\Theta$
\STATE
Find $\bm{d}_n^G \in \mathbb{R}^\Theta$ that solves $\mathsf{H}_n^G \bm{d} = - \hat{\bm{m}}_n^G$
\STATE 
$\bm{\theta}_{n+1} := \bm{\theta}_n + \alpha_n^G \bm{d}^{G}_n$
\ENDLOOP

\LOOP
\STATE 
$\bm{m}_n^D := \beta_1^D \bm{m}_{n-1}^D + (1 - \beta_1^D) \nabla L_{D,\mathcal{R}_n}(\bm{w}_n)$
\STATE
$\hat{\bm{m}}_n^D := (1 - \gamma^{{D}^{n+1}})^{-1} \bm{m}_n^D$
\STATE
$\mathsf{H}_n^D \in \mathbb{S}_{++}^W \cap \mathbb{D}^W$
\STATE
Find $\bm{d}_n^D \in \mathbb{R}^W$ that solves $\mathsf{H}_n^D \bm{d} = - \hat{\bm{m}}_n^D$
\STATE 
$\bm{w}_{n+1} := \bm{w}_n + \alpha_n^D \bm{d}^{D}_n$
\ENDLOOP
\STATE 
$n \gets n+1$ 
\ENDLOOP
\end{algorithmic}
\end{algorithm}

In order to analyze Algorithm \ref{algo:1}, we will assume the following conditions: 

\begin{assum}\label{assum:1}
{\em
\text{}
\begin{enumerate}
\item[(A1)] $\mathsf{H}_n^G = \mathsf{diag}(h_{n,i}^G)$ depends on $\xi_{[n]}^G$ and $\mathsf{H}_n^D = \mathsf{diag}(h_{n,i}^D)$ depends on $\xi_{[n]}^D$. Moreover, $h_{n+1,i}^G \geq h_{n,i}^G$ and $h_{n+1,j}^D \geq h_{n,j}^D$ hold for all $n\in\mathbb{N}$, all $i \in [\Theta]$, and all $j \in [W]$.
\item[(A2)] For all $i \in [\Theta]$, there exists a positive number $H_i^G$ such that $\sup_{n\in\mathbb{N}} \mathbb{E}[h_{n,i}^G] \leq H_i^G$. For all $j \in [W]$, there exists a positive number $H_j^D$ such that $\sup_{n\in\mathbb{N}} \mathbb{E}[h_{n,j}^D] \leq H_j^D$. 
\end{enumerate}
}
\end{assum}

Examples of $\mathsf{H}_n^G \in \mathbb{S}_{++}^\Theta \cap \mathbb{D}^\Theta$ and $\mathsf{H}_n^D \in \mathbb{S}_{++}^W \cap \mathbb{D}^W$ satisfying Assumption \ref{assum:1} are listed in Table \ref{table:ex} (see Appendix \ref{a2}). By referring to the results in \citep{chen2019,iiduka2021,ada}, we can check that $\mathsf{H}_n^G$ and $\mathsf{H}_n^D$ in Table \ref{table:ex} satisfy (A1) and (A2).

\section{Main Results}
\label{sec:3}
\subsection{Convergence analysis of Algorithm \ref{algo:1} using constant learning rates}
\label{subsec:3.1}
We will assume the following conditions:

\begin{assum}\label{assum:c}
{\em 
\text{}
\begin{enumerate}
\item[(C1)] $\alpha_n^G := \alpha^G$ and $\alpha_n^D := \alpha^D$ for all $n\in\mathbb{N}$.
\item[(C2)] There exist positive numbers $M_G$ and $M_D$ such that $\mathbb{E}[\| \nabla_{\bm{\theta}} L_G (\bm{\theta}_n, \bm{w}_n) \|^2] \leq M_G^2$ and $\mathbb{E}[\| \nabla_{\bm{w}} L_D (\bm{\theta}_n, \bm{w}_n) \|^2] \leq M_D^2$.
\item[(C3)] For all $\bm{\theta} = (\theta_i) \in \mathbb{R}^\Theta$ and all $\bm{w} = (w_i) \in \mathbb{R}^W$, there exist positive numbers $\mathrm{Dist}(\bm{\theta})$ and $\mathrm{Dist}(\bm{w})$ such that $\max_{i\in [\Theta]} \sup \{ (\theta_{n,i} - \theta_i)^2 \colon n\in\mathbb{N}\} \leq \mathrm{Dist}(\bm{\theta})$ and $\max_{i\in [W]} \sup \{ (w_{n,i} - w_i)^2 \colon n\in\mathbb{N}\} \leq \mathrm{Dist}(\bm{w})$.
\end{enumerate}
}
\end{assum}

A previous study \citep{Heusel2017} used a decaying learning rate $\mathcal{O}(n^{-\tau})$, where $\tau \in (0,1]$, for TTUR based on Adam to train GANs. This paper investigates the performance of TTUR based on adaptive methods using {\em constant} learning rates defined by (C1). Condition (C2) provides upper bounds on the performance measures $(1/N) \sum_{n=1}^N \mathbb{E} \left[ \langle \bm{\theta}_n - \bm{\theta}, \nabla_{\bm{\theta}} L_G (\bm{\theta}_n, \bm{w}_n) \rangle \right]$ and $(1/N) \sum_{n=1}^N \mathbb{E} \left[ \langle \bm{w}_n - \bm{w}, \nabla_{\bm{w}} L_D (\bm{\theta}_n, \bm{w}_n) \rangle \right]$ (see (\ref{main}) and Theorem \ref{thm:1} for details). (C2) has also been used to analyze adaptive methods for training deep neural networks (see, e.g., \citep{chen2019,ada}). Condition (C3) has been used to provide upper bounds on the performance measures and for analyzing both convex and nonconvex optimization in deep neural networks (see, e.g., \citep{adam,reddi2018,ada}). See Appendix \ref{c3} for remarks regarding (C3).

The following is a convergence analysis of Algorithm \ref{algo:1} (The proof of Theorem \ref{thm:1} is in Appendices \ref{a5} and \ref{a6}).

\begin{thm}\label{thm:1}
Suppose that Assumptions \ref{assum:0}, \ref{assum:1}, and \ref{assum:c} hold and consider the sequence $((\bm{\theta}_{n}, \bm{w}_{n}))_{n\in \mathbb{N}}$ generated by Algorithm \ref{algo:1}. Then, the following hold:

{\em (i)} For all $\bm{\theta} \in \mathbb{R}^\Theta$ and all $\bm{w} \in \mathbb{R}^W$,
\begin{align*}
&\liminf_{n \to + \infty}
\mathbb{E} \left[ \langle \bm{\theta}_n - \bm{\theta}, \nabla_{\bm{\theta}}
L_G (\bm{\theta}_n, \bm{w}_n) \rangle \right]\\
&\leq
\frac{\alpha^{G} (\sigma_G^2 b^{-1} + M_G^2)}{2\tilde{\beta_1^G} \tilde{\gamma}^{G^2} h_{0,*}^G}
+ \sqrt{\Theta \mathrm{Dist}(\bm{\theta})
\left( \frac{\sigma_G^2}{b} + M_G^2 \right)}
\frac{\beta^G}{\tilde{\beta_1^G}}, \\
&\liminf_{n\to + \infty} \mathbb{E}\left[ \langle \bm{w}_n - \bm{w},\nabla_{\bm{w}} L_D (\bm{\theta}_n, \bm{w}_n) \rangle \right]\\
&\leq 
\frac{\alpha^{D} (\sigma_D^2 b^{-1} + M_D^2)}{2\tilde{\beta_1^D} \tilde{\gamma}^{D^2} h_{0,*}^D}
+ \sqrt{W \mathrm{Dist}(\bm{w})
\left( \frac{\sigma_D^2}{b} + M_D^2 \right)}
\frac{\beta_1^D}{\tilde{\beta_1^D}}, 
\end{align*}
where $\tilde{\beta_1^G} := 1 - \beta_1^G$, $\tilde{\beta_1^D} := 1 - \beta_1^D$, $\tilde{\gamma}^G := 1 - \gamma^G$, $\tilde{\gamma}^D := 1 - \gamma^D$, $h_{0,*}^G := \min_{i\in [\Theta]}h_{0,i}^G$, and $h_{0,*}^D := \min_{j\in [W]}h_{0,j}^D$. Furthermore, there exist accumulation points $(\bm{\theta}^*, \bm{w}^*)$ and $(\bm{\theta}_*, \bm{w}_*)$ of $((\bm{\theta}_{n}, \bm{w}_{n}))_{n\in \mathbb{N}}$ such that
\begin{align*}
&\mathbb{E}\left[ \|\nabla_{\bm{\theta}} L_G (\bm{\theta}^*, \bm{w}^*) \|^2 \right]\\
&\leq 
\frac{\alpha^{G} (\sigma_G^2 b^{-1} + M_G^2)}{2\tilde{\beta_1^G} \tilde{\gamma}^{G^2} h_{0,*}^G}
+ \sqrt{\Theta \mathrm{Dist}(\tilde{\bm{\theta}})
\left( \frac{\sigma_G^2}{b} + M_G^2 \right)}
\frac{\beta_1^G}{\tilde{\beta_1^G}},\\
&\mathbb{E}\left[ \|\nabla_{\bm{w}} L_D (\bm{\theta}_*, \bm{w}_*) \|^2 \right]\\
&\leq 
\frac{\alpha^{D} (\sigma_D^2 b^{-1} + M_D^2)}{2\tilde{\beta_1^D} \tilde{\gamma}^{D^2} h_{0,*}^D} 
+ \sqrt{W \mathrm{Dist}(\tilde{\bm{w}})
\left( \frac{\sigma_D^2}{b} + M_D^2 \right)}
\frac{\beta_1^D}{\tilde{\beta_1^D}},
\end{align*}
where $\tilde{\bm{\theta}} := \bm{\theta}^* - \nabla_{\bm{\theta}} L_G (\bm{\theta}^*, \bm{w}^*)$ and $\tilde{\bm{w}} := \bm{w}_* - \nabla_{\bm{w}} L_D (\bm{\theta}_*, \bm{w}_*)$.

{\em (ii)} For all $\bm{\theta} \in \mathbb{R}^\Theta$, all $\bm{w}\in \mathbb{R}^W$, and all $N\geq 1$,
\begin{align*}
\quad
&\frac{1}{N} \sum_{n\in [N]} \mathbb{E} \left[ \langle \bm{\theta}_n - \bm{\theta}, \nabla_{\bm{\theta}}
L_G (\bm{\theta}_n, \bm{w}_n) \rangle \right] \\
&\quad \leq
\frac{\Theta \mathrm{Dist}(\bm{\theta}) H^G}{2 \alpha^G \tilde{\beta_1^G} N}
+
\frac{\alpha^G}{2 \tilde{\beta_1^G} \tilde{\gamma}^{G^2} h_{0,*}^G} 
\left( \frac{\sigma_G^2}{b} + M_G^2 \right)\\
&\quad \quad +
\frac{\beta_1^G}{\tilde{\beta_1^G}}
\sqrt{\Theta \mathrm{Dist}(\bm{\theta})
\left( \frac{\sigma_G^2}{b} + M_G^2 \right)},\\
&\frac{1}{N} \sum_{n\in [N]} \mathbb{E} \left[ \langle \bm{w}_n - \bm{w}, \nabla_{\bm{w}}
L_D (\bm{\theta}_n, \bm{w}_n) \rangle \right]\\
&\quad \leq
\frac{W \mathrm{Dist}(\bm{w}) H^D}{2 \alpha^D \tilde{\beta_1^D} N}
+
\frac{\alpha^D}{2 \tilde{\beta_1^D} \tilde{\gamma}^{D^2} h_{0,*}^D} 
\left( \frac{\sigma_D^2}{b} + M_D^2 \right)\\
&\quad \quad+
\frac{\beta_1^D}{\tilde{\beta_1^D}}
\sqrt{W \mathrm{Dist}(\bm{w})
\left( \frac{\sigma_D^2}{b} + M_D^2 \right)},
\end{align*}
where $H^G := \max_{i\in [\Theta]} H_i^G$ and $H^D := \max_{j\in [W]} H_j^D$.
\end{thm}

Theorem \ref{thm:1}(i) and (ii) indicate that the larger the batch size $b$ is, the smaller the upper bounds of the performance measures become. From Theorem \ref{thm:1}(i), it would be desirable to use small learning rates $\alpha^G$ and $\alpha^D$. Meanwhile, Theorem \ref{thm:1}(ii) indicates that there is no evidence that using sufficiently small $\alpha^G$ and $\alpha^D$ is good for training GANs since the upper bounds in Theorem \ref{thm:1} depend on $\alpha^G$, $\alpha^D$, $1/\alpha^G$, and $1/\alpha^D$. Indeed, the results reported in \citep{Heusel2017} used small $\alpha^G, \alpha^D=10^{-4}, 10^{-5}$. 

\subsection{Relationship between batch size and number of steps for Algorithm \ref{algo:1}}
\label{subsec:3.2}
The relationship between $b$ and the number of steps $N$ satisfying an $\epsilon$--approximation of TTUR is as follows (The proof of Theorem \ref{thm:2} is in Appendix \ref{a7}):

\begin{thm}\label{thm:2}
Suppose that Assumptions \ref{assum:0}, \ref{assum:1}, and \ref{assum:c} hold and consider Algorithm \ref{algo:1}. Then, $N_G$ and $N_D$ defined by
\begin{align}\label{lower}
\begin{split}
&N_G (b) := \frac{A_G b}{(\epsilon_G^2 - C_G)b -B_G} \leq N_G 
\text{ for } b > \frac{B_G}{\epsilon_G^2 - C_G},\\
&N_D (b) := \frac{A_D b}{(\epsilon_D^2 - C_D)b -B_D} \leq N_D
\text{ for } b > \frac{B_D}{\epsilon_D^2 - C_D}
\end{split}
\end{align}
satisfy 
\begin{align}\label{approximation}
\begin{split}
&\frac{1}{N_G} \sum_{n=1}^{N_G} \mathbb{E} \left[ \langle \bm{\theta}_n - \bm{\theta}, \nabla_{\bm{\theta}}
L_G (\bm{\theta}_n, \bm{w}_n) \rangle \right] \leq \epsilon_G^2, \\
&\frac{1}{N_D} \sum_{n=1}^{N_D} \mathbb{E} \left[ \langle \bm{w}_n - \bm{w}, \nabla_{\bm{w}}
L_D (\bm{\theta}_n, \bm{w}_n) \rangle \right] \leq \epsilon_D^2,
\end{split}
\end{align}
where $A_G$, $B_G$, $C_G$, $A_D$, $B_D$, and $C_D$ are defined as in (\ref{main}). Moreover, the functions $N_G(b)$ and $N_D(b)$ defined by (\ref{lower}) are monotone decreasing and convex for $b > B_G/(\epsilon_G^2 - C_G)$ and $b > B_D/(\epsilon_D^2 - C_D)$. 
\end{thm} 

Theorem \ref{thm:2} indicates that $N_G$ and $N_D$ defined by (\ref{lower}) decrease as the batch size increases. Accordingly, it is useful to choose a sufficiently large $b$ in the sense of minimizing of the number of steps $N$ needed for an $\epsilon$--approximation of TTUR.

\subsection{Existence of a critical batch size}
A particular concern is how large $b$ should be. Here, we consider SFO complexities defined for the number of steps needed for (\ref{approximation}) and for the batch size by 
\begin{align}\label{SFO}
\begin{split}
&N_G(b) b = \frac{A_G b^2}{(\epsilon_G^2 - C_G)b -B_G}\text{ and} \\
&N_D(b) b = \frac{A_D b^2}{(\epsilon_D^2 - C_D)b - B_D}.
\end{split}
\end{align} 

The following theorem guarantees the existence of critical batch sizes that are global minimizers of $N_G(b)b$ and $N_D(b)b$ defined by (\ref{SFO}) (The proof of Theorem \ref{thm:3} is in Appendix \ref{a8}).

\begin{thm}\label{thm:3}
Suppose that Assumptions \ref{assum:0}, \ref{assum:1}, and \ref{assum:c} hold and consider Algorithm \ref{algo:1}. Then, there exist 
\begin{align}\label{star}
b_G^\star := \frac{2 B_G}{\epsilon_G^2 - C_G} \text{ and }
b_D^\star := \frac{2 B_D}{\epsilon_D^2 - C_D}
\end{align}
such that $b_G^\star$ minimizes the convex function $N_G(b) b$ ($b > B_G/(\epsilon_G^2 - C_G)$) and $b_D^\star$ minimizes the convex function $N_D(b) b$ ($b > B_D/(\epsilon_D^2 - C_D)$).
\end{thm} 
Theorem \ref{thm:3} leads to the following proposition that gives lower bounds for the critical batch sizes.

\setcounter{prop}{3}
\begin{prop}\label{prop:4}
Suppose that the assumptions in Theorem \ref{thm:3} hold and consider Algorithm \ref{algo:1}. Then, $b_G^\star$ and $b_D^\star$ defined by (\ref{star}) satisfy the following that \\
{\em (i)} for Adam,
\begin{align*}
b_G^\star &\geq \frac{\sigma_G^2 \alpha^G}{\epsilon_G^3 (1-\beta_1^G)^3 \sqrt{\frac{\Theta}{1-\beta_2^G} \frac{1}{|S|^2}}}\text{ and } \\
b_D^\star &\geq \frac{\sigma_D^2 \alpha^D}{\epsilon_D^3 (1-\beta_1^D)^3 \sqrt{\frac{W}{1-\beta_2^D} \frac{1}{|S|^2}}},
\end{align*}
{\em (ii)} for AdaBelief,
\begin{align*}
b_G^\star &\geq \frac{\sigma_G^2 \alpha^G}{\epsilon_G^3 (1-\beta_1^G)^3 \sqrt{\frac{4 \Theta}{1-\beta_2^G} \frac{1}{|S|^2}}} \text{ and } \\
b_D^\star &\geq \frac{\sigma_D^2 \alpha^D}{\epsilon_D^3 (1-\beta_1^D)^3 \sqrt{\frac{4W}{1-\beta_2^D} \frac{1}{|S|^2}}} ,
\end{align*}
{\em (iii)} for RMSProp,
\begin{align*}
b_G^\star \geq \frac{\sigma_G^2 \alpha^G}{\epsilon_G^3 \sqrt{\frac{\Theta}{|S|^2}}} \text{ and }
b_D^\star \geq \frac{\sigma_D^2 \alpha^D}{\epsilon_D^3 \sqrt{\frac{W}{|S|^2}}},
\end{align*}
where $\sigma_G^2, \sigma_D^2 \geq 0$, $\alpha^G, \alpha^D, \epsilon_G, \epsilon_D > 0$, $\beta_1^G, \beta_1^D \in [0,1)$, and $\beta_2^G, \beta_2^D \in [0,1)$.
\end{prop} 
Theorem \ref{thm:3} indicates that critical batch sizes exist in the sense of minimizing the SFO complexities $N_G(b) b$ and $N_D(b) b$. We are interested in verifying whether or not a critical batch size exists for training GANs as is the case of training deep neural networks \cite{shallue2019,zhang2019,iiduka2022}. The next section numerically examines the relationship between the batch size $b$ and the number of steps $N$ and also that between $b$ and the SFO complexity $N b$ to see if there is a critical batch size $b^\star$ at which $N(b) b$ is minimized. Proposition \ref{prop:4} indicates that a lower bound for the critical batch size can be estimated from some hyperparameters. Hence, we would like to check whether the estimated sizes are close to the measured ones (see Section \ref{subsec:4.4}).

\section{Numerical Results}
\label{sec:4}
We measured the number of steps required to achieve a low FID \citep{Heusel2017} score for different batch sizes in several GAN trainings. The experimental environment consisted of NVIDIA DGX A100$\times$8GPU and Dual AMD Rome7742 2.25-GHz, 128 Cores$\times$2CPU. The software environment was Python 3.8.2, Pytorch 1.6.0, and CUDA 11.6. The distribution of $\xi_n^G$ was a uniform one. The clean-fid package \cite{clean-fid} was used to calculate the FID. The code is available at \url{https://github.com/iiduka-researches/GANs}. The TTURs based on Adam, AdaBelief, and RMSProp used $\mathsf{H}_n^G$ and $\mathsf{H}_n^D$ defined in Table \ref{table:ex}. See Table \ref{table:hyper} for the optimizer hyperparameters used in the experiment. The learning rate used in each experiment was determined on the basis of a grid search of 36 combinations of the generator learning rate $\alpha^G$ and discriminator learning rate $\alpha^D$ (see Figure \ref{fig:1} in Appendix \ref{a1}). Appendix \ref{a3} indicates that, when a fixed batch size is used, the FID scores of the TTURs used in the experiments decrease sufficiently as the number of steps increases. 

\subsection{Training DCGAN on the LSUN-Bedroom dataset}
\label{subsec:4.1}
First, we evaluated the performance of TTURs based on RMSProp, AdaBelief, and Adam in training DCGAN \citep{radford2016unsupervised} on the LSUN-Bedroom dataset. Figure \ref{fig:2} plots the number of steps $N$ needed to achieve an FID score lower than 70 versus the batch size $b$. The figure indicates that the number of steps for TTUR based on any optimizer is monotone decreasing and convex with respect to $b$. Figure \ref{fig:3} plots the SFO complexity $Nb$ versus $b$. The figure indicates that $Nb$ for TTUR based on any optimizer is convex with respect to $b$. 

\subsection{Training WGAN-GP on the CelebA dataset}
\label{subsec:4.2}
Next, we evaluated the performance of TTURs based on RMSProp, AdaBelief, and Adam in training WGAN-GP on the CelebA dataset. The original WGAN-GP code updates the discriminator five times for each generator update, whereas the discriminator is updated only once when using TTUR. Figure \ref{fig:4} plots the number of steps $N$ needed to achieve an FID score lower than 50 versus the batch size $b$. The figure indicates that the number of steps for TTUR based on any optimizer is a monotone decreasing and convex function of $b$. Figure \ref{fig:5} plots $Nb$ versus $b$. The figure indicates that $Nb$ for TTUR based on any optimizer is a convex function of $b$. 

\begin{table}[htbp]
\begin{minipage}[t]{0.5\textwidth}
\centering
\caption{Parameters used to train GANs}
\label{table:3}
\begin{tabular}{cccc}
\hline
&Section \ref{subsec:4.1} &Section \ref{subsec:4.2} &Section {\ref{app:big}}\\
\hline
$\Theta$ &$\sepnum{.}{,}{,}{3576704}$ &$\sepnum{.}{,}{,}{3576704}$ &$\sepnum{.}{,}{,}{70433795}$\\
W &$\sepnum{.}{,}{,}{2765568}$ &$\sepnum{.}{,}{,}{2765568}$ &$\sepnum{.}{,}{,}{87982369}$\\
$|S|$ &$\sepnum{.}{,}{,}{3033042}$ &$\sepnum{.}{,}{,}{162770}$ &$\sepnum{.}{,}{,}{1281167}$\\
\hline
\end{tabular}
\end{minipage}
\end{table}

\begin{figure}[htbp]
\begin{tabular}{c}
\begin{minipage}[t]{0.99\hsize}
\centering
\includegraphics[width=1\textwidth]{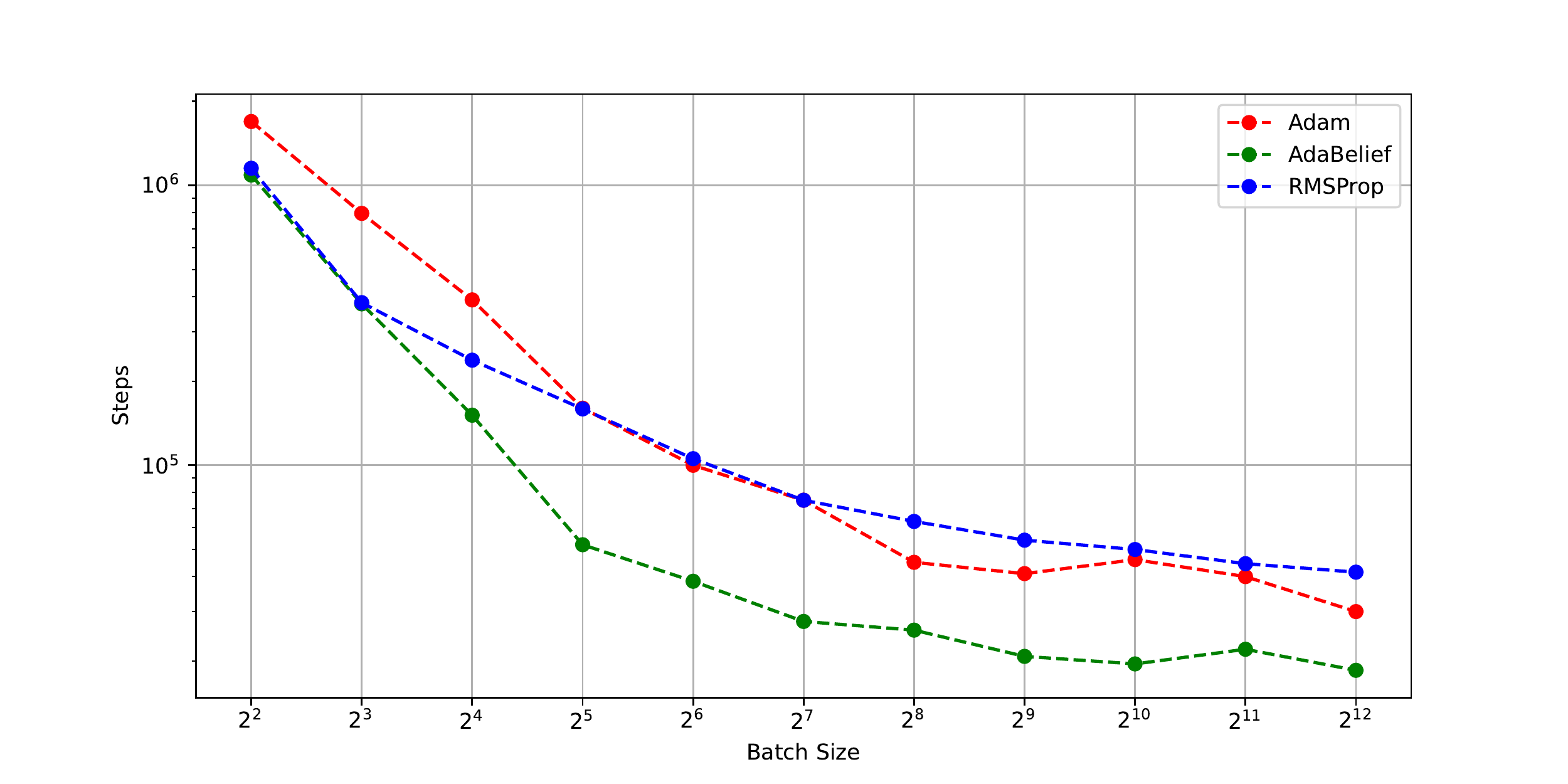}
\caption{Number of steps for TTURs based on Adam, AdaBelief, and RMSProp versus batch size needed to train DCGAN on the LSUN-Bedroom dataset. The average of multiple runs is plotted.}
\label{fig:2}
\end{minipage}
\end{tabular}
\end{figure}

\begin{figure}[htbp]
\begin{tabular}{c}
\begin{minipage}[t]{0.99\hsize}
\centering
\includegraphics[width=1\textwidth]{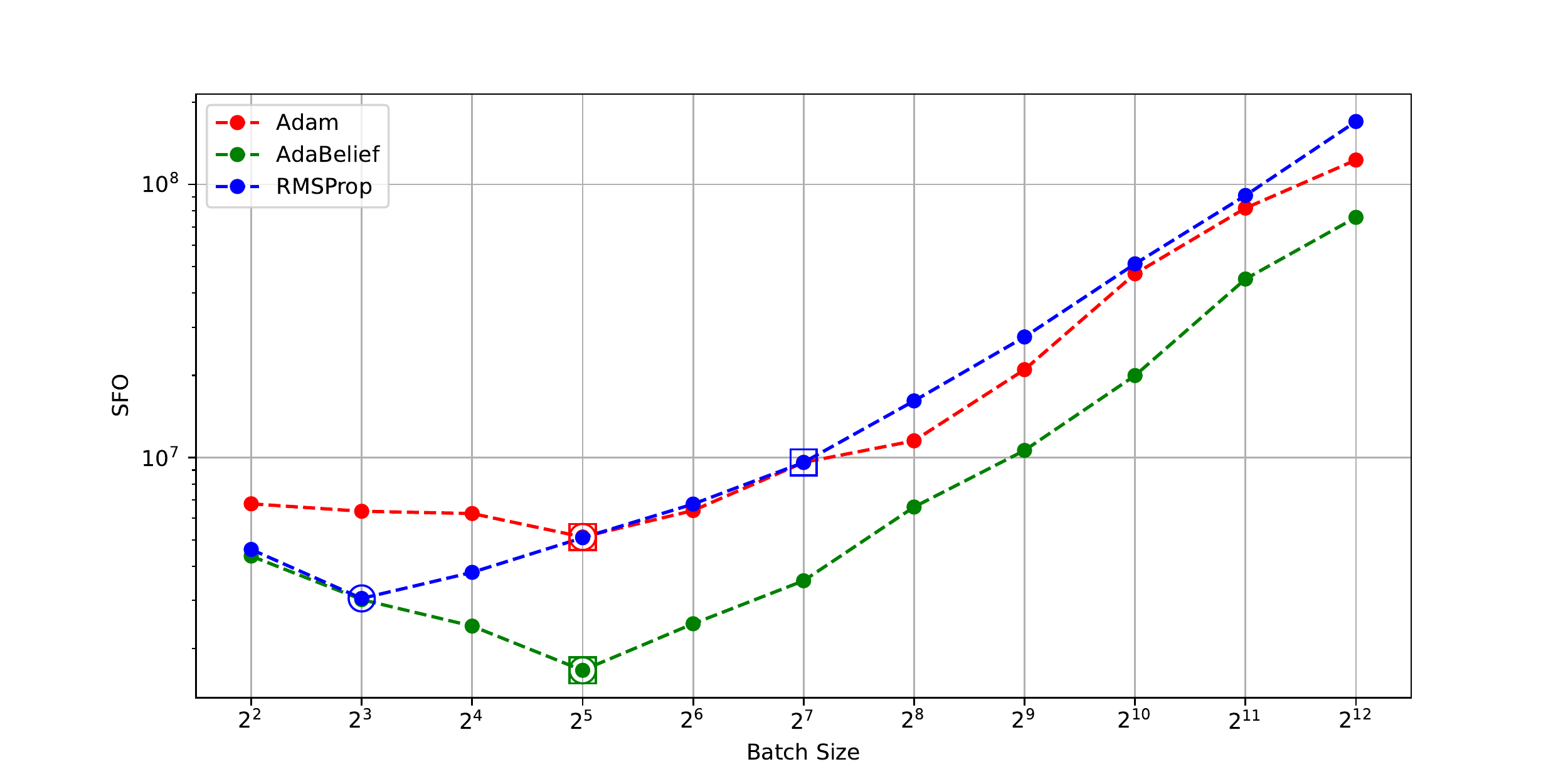}
\caption{SFO complexities for TTURs based on Adam, AdaBelief, and RMSProp versus batch size needed to train DCGAN on the LSUN-Bedroom dataset. The double circle symbol denotes the measured critical batch size that minimizes SFO complexity. The square symbol denotes the estimated critical batch size.}
\label{fig:3}
\end{minipage} 
\end{tabular}
\end{figure}

\begin{figure}[htbp]
\begin{tabular}{c}
\begin{minipage}[t]{0.99\hsize}
\centering
\includegraphics[width=1\textwidth]{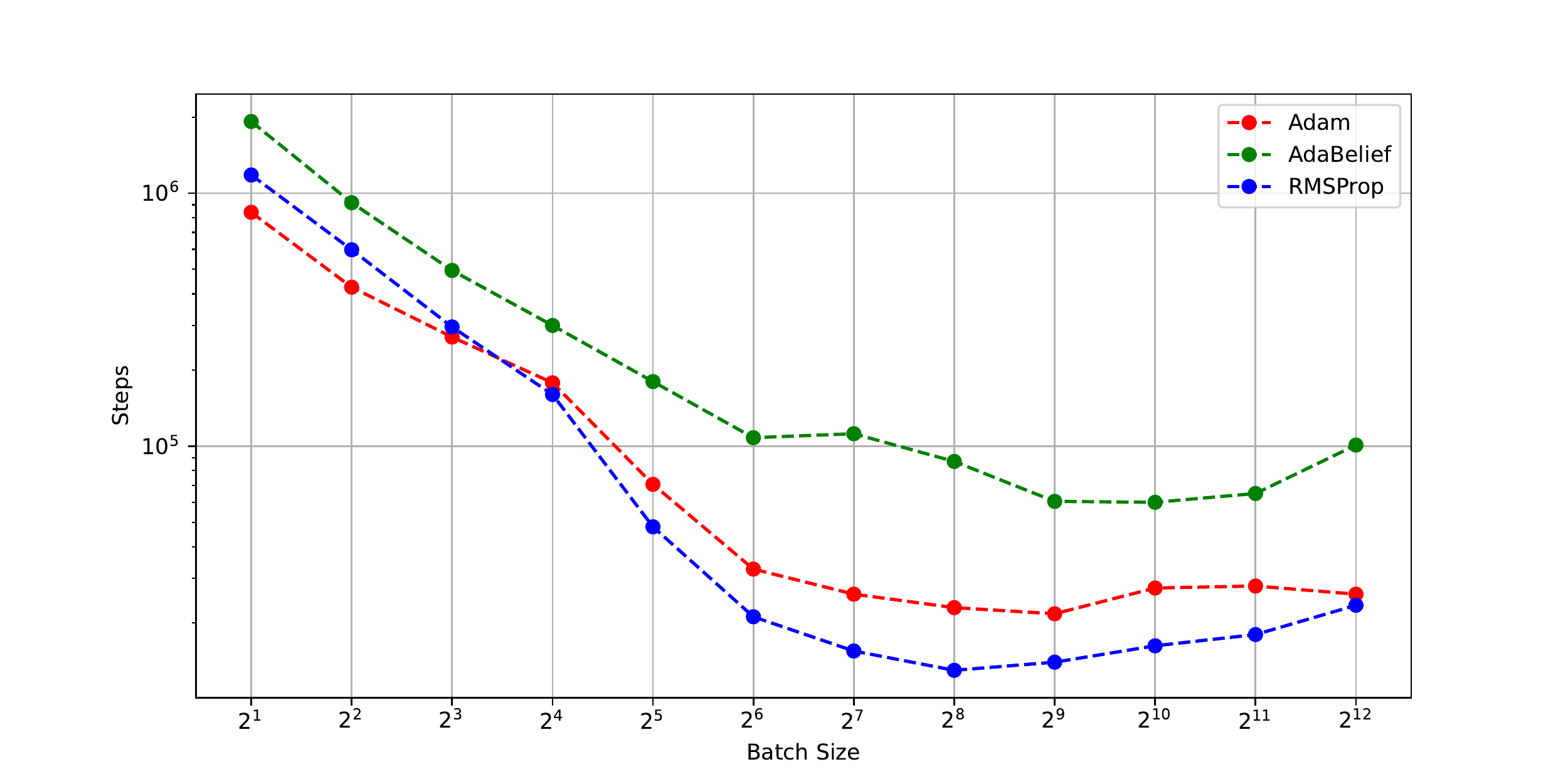}
\caption{Number of steps for TTURs based on Adam, AdaBelief, and RMSProp versus batch size needed to train WGAN-GP on the CelebA dataset. The average of multiple runs is plotted.}
\label{fig:4}
\end{minipage}
\end{tabular}
\end{figure}

\begin{figure}[htbp]
\begin{tabular}{c}
\begin{minipage}[t]{0.99\hsize}
\centering
\includegraphics[width=1\textwidth]{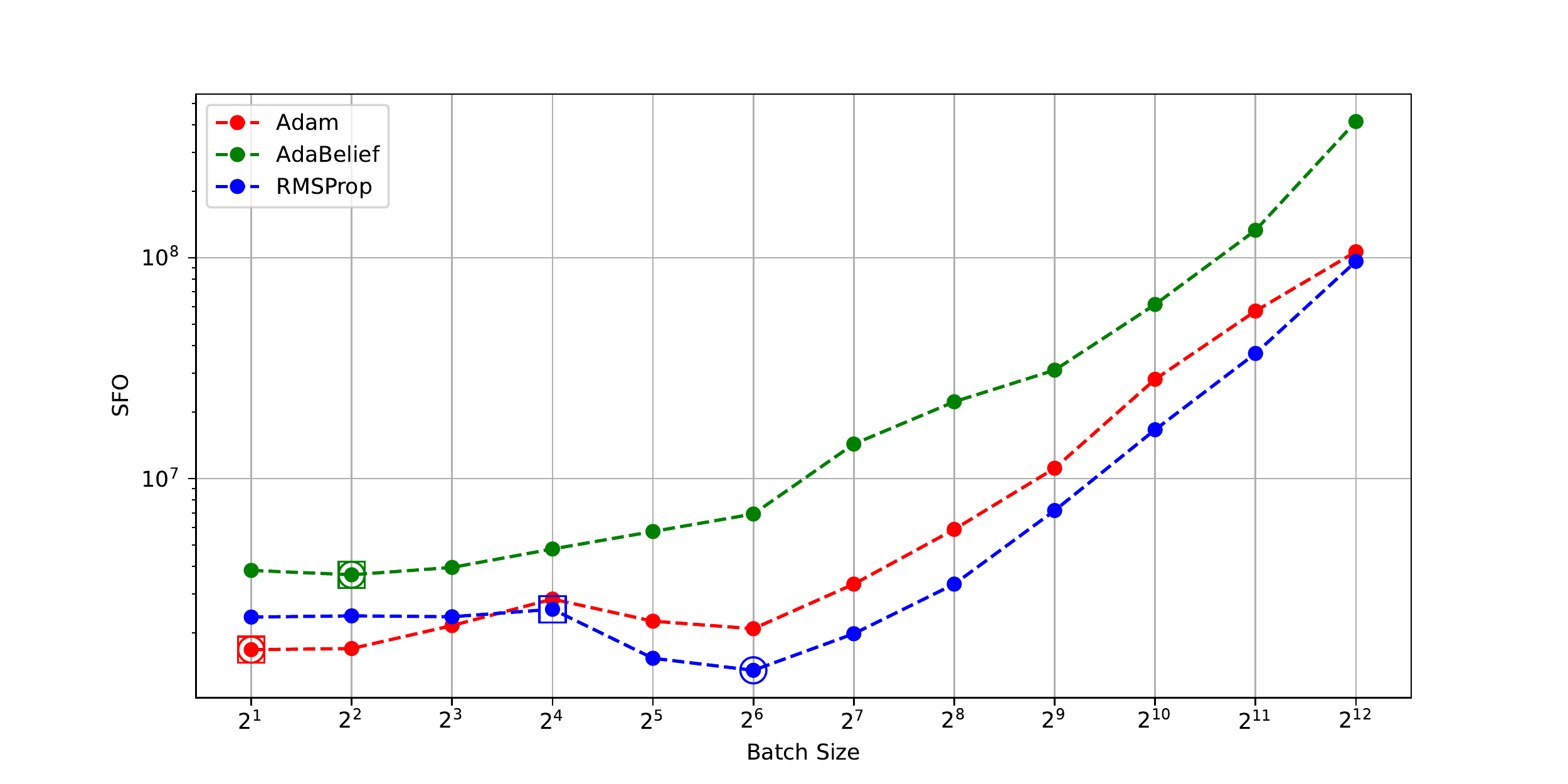}
\caption{SFO complexities for TTURs based on Adam, AdaBelief, and RMSProp versus batch size needed to train WGAN-GP on the CelebA dataset. The double circle symbol denotes the measured critical batch size that minimizes SFO complexity. The square symbol denotes the estimated critical batch size.}
\label{fig:5}
\end{minipage} 
\end{tabular}
\end{figure}

\begin{figure}[htbp]
\begin{tabular}{c}
\begin{minipage}[t]{0.99\hsize}
\centering
\includegraphics[width=1\textwidth]{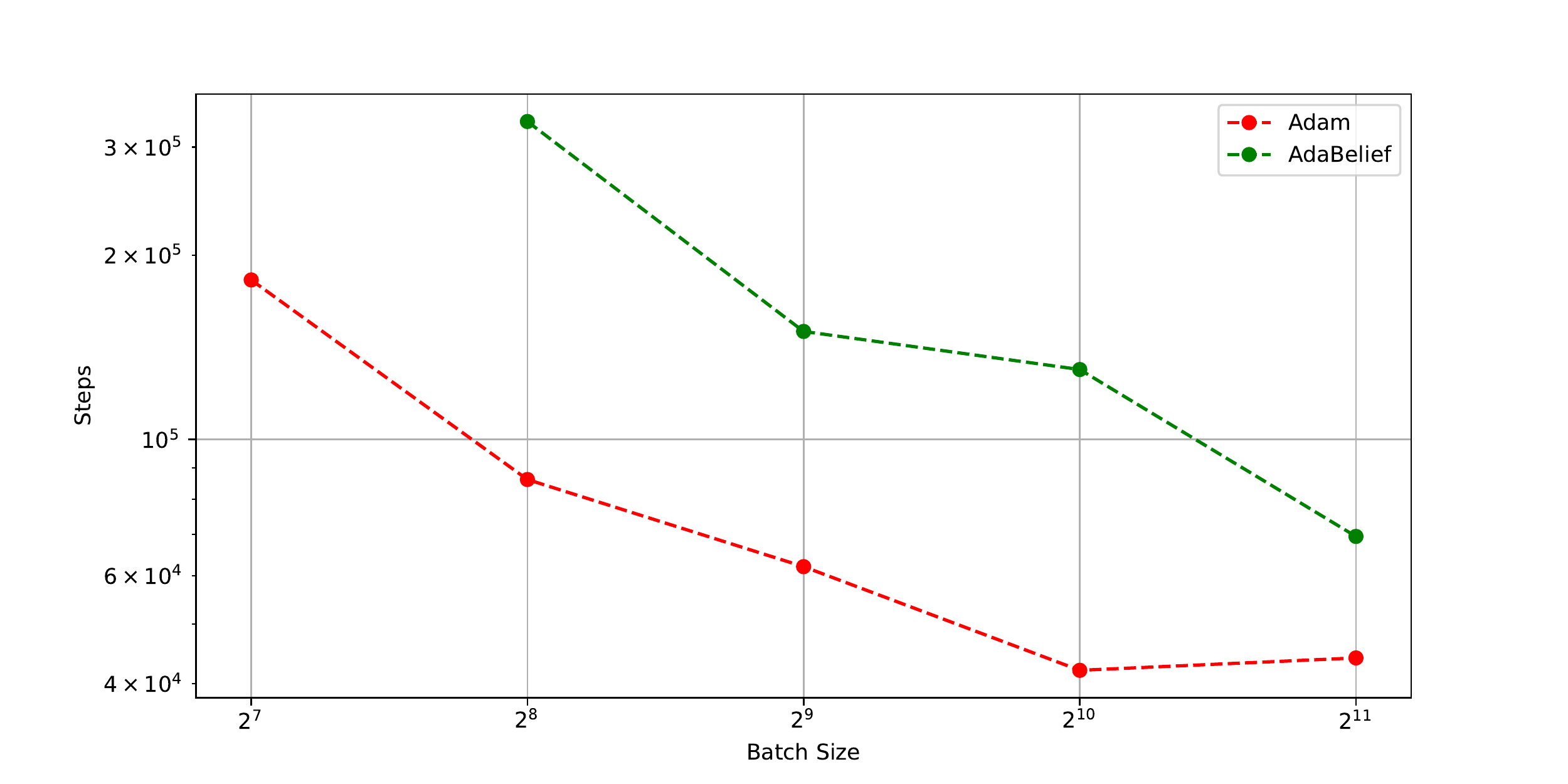}
\caption{Number of steps for TTURs based on Adam and AdaBelief versus batch size needed to train BigGAN on the ImageNet dataset. The average of multiple runs is plotted.}
\label{fig:big1}
\end{minipage}
\end{tabular}
\end{figure}

\begin{figure}[htbp]
\begin{tabular}{c}
\begin{minipage}[t]{0.99\hsize}
\centering
\includegraphics[width=1\textwidth]{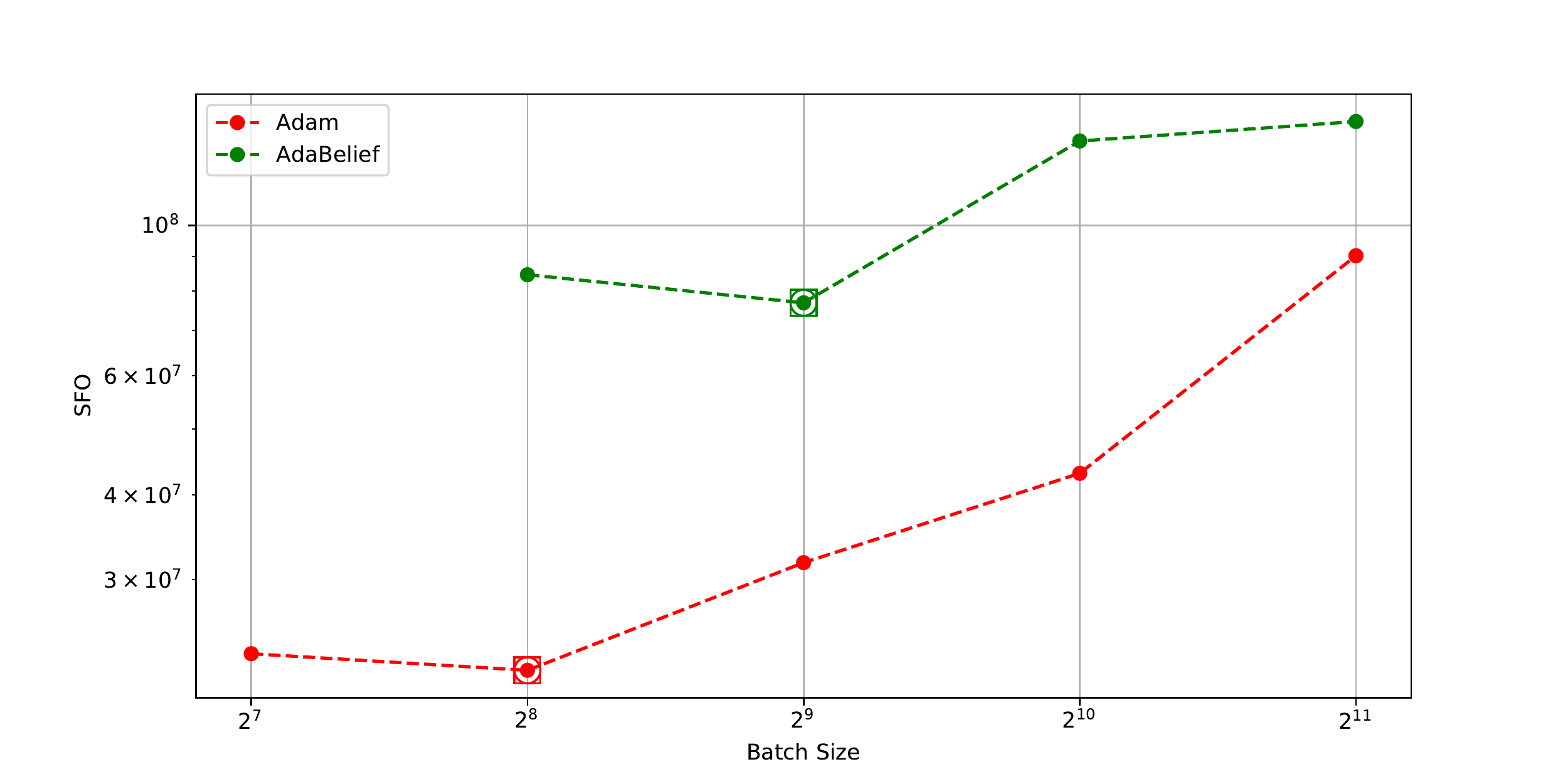}
\caption{SFO complexities for TTURs based on Adam and AdaBelief versus batch size needed to train BigGAN on the ImageNet dataset. The double circle symbol denotes the measured critical batch size that minimizes SFO complexity. The square symbol denotes the estimated critical batch size.}
\label{fig:big2}
\end{minipage} 
\end{tabular}
\end{figure}

\begin{table*}
\begin{minipage}[t]{0.99\textwidth}
\caption{Measured and estimated critical batch sizes}
\label{table:4}
\centering
\begin{tabular}{ccc|cc|cc}
\hline
&\multicolumn{2}{c|}{Section \ref{subsec:4.1}} &\multicolumn{2}{c|}{Section \ref{subsec:4.2}} &\multicolumn{2}{c}{Section \ref{app:big}}\\
\cline{2-3}
\cline{4-5}
\cline{6-7}
&measured &estimated &measured &estimated &measured &estimated\\
\hline
Adam &$2^5$ &$2^5$ &$2^1$ &$2^1$ &$2^8$ &$2^8$\\
AdaBelief &$2^5$ &$2^5$ &$2^2$ &$2^2$ &$2^9$ &$2^9$\\
RMSProp &$2^3$ &$2^7$ &$2^6$ &$2^4$ &- &-\\
\hline
\end{tabular}
\end{minipage}
\end{table*}

\subsection{Training BigGAN on the ImageNet dataset}
\label{app:big}
We evaluated the performance of TTURs based on AdaBelief and Adam in training BigGAN \citep{brock2018large} on the ImageNet dataset. Figure \ref{fig:big1} plots the number of steps $N$ needed to achieve an FID score lower than 25 versus the batch size $b$. The figure indicates that the number of steps for TTUR based on AdaBelief and Adam is monotone decreasing and convex with respect to $b$. Figure \ref{fig:big2} plots the SFO complexity $Nb$ versus $b$. The figure indicates that $Nb$ for TTUR based on AdaBelief and Adam is convex with respect to $b$.

We can conclude that the results in Sections \ref{subsec:4.1}, \ref{subsec:4.2}, and \ref{app:big} support the theoretical results (Theorems \ref{thm:2} and \ref{thm:3}).

\subsection{Estimation of lower bound of critical batch size} 
\label{subsec:4.4}
Proposition \ref{prop:4} indicates that a lower bound on the critical batch size can be estimated with some parameters. The parameters used in the experiments are shown in Table \ref{table:3}; see Section \ref{a:hyper} for the settings of $\alpha^G$, $\beta_1^G$, and so on. Figure \ref{fig:3} indicates that the measured critical batch sizes minimizing the SFO complexities of Adam, AdaBelief, and RMSProp are $2^{5} = 32$, $2^{5}=32$, and $2^{3}=8$, respectively. For batch sizes below $2^{7}=128$, there is no significant difference in SFO complexity, but it is clear that the measured critical batch size is less than $2^7$. In the previous studies \cite{shallue2019,zhang2019}, the critical batch sizes for training deep neural networks are large, such as $2^{12}$, while the critical batch sizes for training GANs are small. According to Figure \ref{fig:3}, in the DCGAN on the LSUN-Bedroom dataset setting, the measured critical batch size for Adam is $2^5$; using this and Proposition \ref{prop:4}(i) to back-calculate $\sigma_G^2$/$\epsilon_G^3$ gives $\sigma_G^2$/$\epsilon_G^3$ $=788.7$. We can use this ratio and Proposition \ref{prop:4}(ii), (iii) to estimate a lower bound of $47.9$ for AdaBelief and $126.5$ for RMSProps (see Table \ref{table:4}).

Figure \ref{fig:5} indicates that the measured critical batch sizes for Adam, AdaBelief, and RMSProp are $2$, $2^{2}=4$, and $2^{6}=64$, respectively. As in Figure \ref{fig:3}, there is no significant difference in SFO complexity for batch sizes less than $2^{6}$, but it is clear that the critical batch size is less than $2^{6}$. As expected, it is smaller than the critical batch sizes of deep neural networks. In the same way as above, the estimated lower bounds on the WGAN-GP on the CelebA dataset are $1.7$ for Adam, $4.25$ for AdaBelief, and $20.3$ for RMSProp (see Table \ref{table:4}).

According to Figure \ref{fig:big2}, in the BigGAN on the ImageNet dataset setting, the measured critical batch size for Adam is $2^8$; using this and Proposition \ref{prop:4}(i) to back calculate $\sigma_G^2$/$\epsilon_G^3$ gives $\sigma_G^2$/$\epsilon_G^3$ $=530303.8$. We can use this ratio and Proposition \ref{prop:4}(ii) to estimate a lower bound of $511.99$ for AdaBelief (see Table \ref{table:4}).

Proposition \ref{prop:4}(i) and (ii) indicate that the estimated critical batch sizes of Adam and AdaBelief strongly depend on the values of $\beta_1$ and $\beta_2$. Accordingly, the estimated critical batch sizes of Adam and AdaBelief were found to be the same as the measured ones. Meanwhile, Proposition \ref{prop:4}(iii) indicates that the critical batch size of RMSProp is completely independent of the values of $\beta_1$ and $\beta_2$. In particular, $\beta_1$ in RMSProp is always 0 and $\beta_2$ in RMSProp is not used to estimate the critical batch size (see the proof of Proposition \ref{prop:4} in Appendix \ref{app:prop4} for details). The independence of $\beta_1$ and $\beta_2$ may have caused the failure in estimating the critical batch size of RMSProp.

\section{Conclusion}
\label{sec:5}
We considered a stationary point problem in a GAN and performed a theoretical analysis of TTUR with constant learning rates to find a solution. We evaluated the upper bound of the expectation of the gradient of the loss function of the discriminator and the generator and showed that it is small when small constant learning rates and a large batch size are used. Next, we examined the relationship between the number of steps needed for solving the problem and batch size and showed that the number of steps decreases as the batch size increases. Moreover, we evaluated the SFO complexity of TTUR to check how large the batch size should be and showed that there is a critical batch size minimizing the SFO complexity, which is a convex function of the batch size. We also showed that it is possible to estimate the critical batch size specific to the model-dataset-optimizer combination. Finally, we provided numerical examples to support our theoretical analyzes. In particular, the numerical results showed that TTUR with small constant learning rates can be used to train DCGAN, WGAN-GP, and BigGAN, the number of steps needed to train them is monotone decreasing with the batch size, a critical batch size that minimizes the SFO complexity exists, and the estimated critical batch size is close to the experimentally measured value for DCGAN, WGAN-GP, and BigGAN. 

\section*{Acknowledgements}
We are sincerely grateful to Program Chairs, Area Chairs, and the three anonymous reviewers for helping us improve the original manuscript. We would like to thank Hiroki Naganuma (Mila, UdeM) for his help on the PyTorch implementation. This research is partly supported by the computational resources of the DGX A100 named TAIHO at Meiji University. This work was supported by the Japan Society for the Promotion of Science (JSPS) KAKENHI Grant Number 21K11773 awarded to Hideaki Iiduka.

\nocite{langley00}

\bibliography{icml2023_conference}
\bibliographystyle{icml2023}

\newpage
\appendix
\onecolumn
\section{Appendix}
\label{appendix:1}
Unless stated otherwise, all relations between random variables are supported to hold almost surely. Let $S \in \mathbb{S}_{++}^d$. The $S$-inner product of $\mathbb{R}^d$ is defined for all $\bm{x}, \bm{y} \in \mathbb{R}^d$ by $\langle \bm{x},\bm{y} \rangle_S := \langle \bm{x}, S \bm{y} \rangle$ and the $S$-norm is defined by $\|\bm{x}\|_S := \sqrt{\langle \bm{x}, S \bm{x} \rangle}$. 
The Hadamard product of $\mathbb{R}^d$ is defined for all 
$\bm{x} = (x_i)_{i=1}^d \in \mathbb{R}^d$ by $\bm{x} \odot \bm{x} := (x_i^2)_{i=1}^d \in \mathbb{R}^d$.

\subsection{Examples of diagonal matrix in Algorithm \ref{algo:1}}
\label{a2}
\begin{table*}[htbp]
\begin{center}
\caption{Examples of $\mathsf{H}_n^G \in \mathbb{S}_{++}^\Theta \cap \mathbb{D}^\Theta$ and $\mathsf{H}_n^D \in \mathbb{S}_{++}^W \cap \mathbb{D}^W$ (steps 6 and 13) in Algorithm \ref{algo:1} ($\beta_2^G, \beta_2^D \in [0,1)$)}
\label{table:ex}
\begin{tabular}{l|ll}
\hline

& $\mathsf{H}_n^G$ 
& $\mathsf{H}_n^D$ \\ \hline \hline

RMSProp
& $\bm{p}_n^G = \nabla L_{G,\mathcal{S}_n}(\bm{\theta}_n) \odot \nabla L_{G,\mathcal{S}_n}(\bm{\theta}_n)$
& $\bm{p}_n^D = \nabla L_{D,\mathcal{R}_n}(\bm{w}_n) \odot \nabla L_{D,\mathcal{R}_n}(\bm{w}_n)$ \\

\citep{rmsprop}
& $\bm{v}_n^G = \beta_2^G \bm{v}_{n-1}^G + (1- \beta_2^G) \bm{p}_n^G$
& $\bm{v}_n^D = \beta_2^D \bm{v}_{n-1}^D + (1- \beta_2^D) \bm{p}_n^D$ \\

($\gamma^G = \gamma^D = \beta_1^G = \beta_1^D$)
& $\mathsf{H}_n^G = \mathsf{diag} (\sqrt{v_{n,i}^G})$ 
& $\mathsf{H}_n^D = \mathsf{diag} (\sqrt{v_{n,i}^D})$\\ 
\hline

Adam 
& $\bm{p}_n^G = \nabla L_{G,\mathcal{S}_n}(\bm{\theta}_n) \odot \nabla L_{G,\mathcal{S}_n}(\bm{\theta}_n)$
& $\bm{p}_n^D = \nabla L_{D,\mathcal{R}_n}(\bm{w}_n) \odot \nabla L_{D,\mathcal{R}_n}(\bm{w}_n)$ \\

\citep{Heusel2017}
& $\bm{v}_n^G = \beta_2^G \bm{v}_{n-1}^G + (1- \beta_2^G) \bm{p}_n^G$
& $\bm{v}_n^D = \beta_2^D \bm{v}_{n-1}^D + (1- \beta_2^D) \bm{p}_n^D$ \\

\citep{adam}
& $\bar{\bm{v}}_n^G = \frac{\bm{v}_n^G}{1- \beta_2^{G^n}}$ 
& $\bar{\bm{v}}_n^D = \frac{\bm{v}_n^D}{1- \beta_2^{D^n}}$\\

($v_{n,i}^G \leq v_{n+1,i}^G$)
& $\mathsf{H}_n^G = \mathsf{diag} (\sqrt{\bar{v}_{n,i}^G})$ 
& $\mathsf{H}_n^D = \mathsf{diag} (\sqrt{\bar{v}_{n,i}^D})$\\ 

($v_{n,i}^D \leq v_{n+1,i}^D$)
& \\
($\gamma^G = \gamma^D = \beta_1^G = \beta_1^D$)
& \\
\hline

AMSGrad \cite{reddi2018} 
& $\bm{p}_n^G = \nabla L_{G, \mathcal{S}_n}(\bm{\theta}_n) \odot \nabla L_{G, \mathcal{S}_n}(\bm{\theta}_n)$ 
& $\bm{p}_n^D = \nabla L_{D, \mathcal{R}_n}(\bm{w}_n) \odot \nabla L_{D, \mathcal{R}_n}(\bm{w}_n)$ \\

\citep{chen2019}
& $\bm{v}_n^G = \beta_2^G \bm{v}_{n-1}^G + (1-\beta_2^G) \bm{p}_n^G$ 
& $\bm{v}_n^D = \beta_2^D \bm{v}_{n-1}^D + (1-\beta_2^D) \bm{p}_n^D$\\

($\gamma^G = \gamma^D = 0$)
& $\hat{\bm{v}}_n^G = (\max \{ \hat{v}_{n-1,i}^G, v_{n,i}^G \})_{i=1}^\Theta$ 
& $\hat{\bm{v}}_n^D = (\max \{ \hat{v}_{n-1,i}^D, v_{n,i}^D \})_{i=1}^W$\\

& $\mathsf{H}_n^G = \mathsf{diag} (\sqrt{\hat{v}_{n,i}^G})$ 
& $\mathsf{H}_n^D = \mathsf{diag} (\sqrt{\hat{v}_{n,i}^D})$ \\ \hline

AdaBelief 
& $\tilde{\bm{p}}_n^G = \nabla L_{G, \mathcal{S}_n}(\bm{\theta}_n) - \bm{m}_n^G$ 
& $\tilde{\bm{p}}_n^D = \nabla L_{D, \mathcal{R}_n}(\bm{w}_n) - \bm{m}_n^D$ \\

\citep{ada} 
& $\tilde{\bm{s}}_n^G = \tilde{\bm{p}}_n^G \odot \tilde{\bm{p}}_n^G$ 
& $\tilde{\bm{s}}_n^D = \tilde{\bm{p}}_n^D \odot \tilde{\bm{p}}_n^D$ \\

($s_{n,i}^G \leq s_{n+1,i}^G$)
& $\bm{s}_n^G = \beta_2^G \bm{v}_{n-1}^G + (1-\beta_2^G) \tilde{\bm{s}}_n^G$ 
& $\bm{s}_n^D = \beta_2^D \bm{v}_{n-1}^D + (1-\beta_2^D) \tilde{\bm{s}}_n^D$\\

($s_{n,i}^D \leq s_{n+1,i}^D$)
& $\hat{\bm{s}}_n^G = \frac{\bm{s}_n^G}{1- \beta_2^{G^n}}$ 
& $\hat{\bm{s}}_n^D = \frac{\bm{s}_n^D}{1- \beta_2^{D^n}}$ \\

($\gamma^G = \gamma^D = \beta_1^G = \beta_1^D$)
& $\mathsf{H}_n^G = \mathsf{diag} (\sqrt{\hat{s}_{n,i}^G})$ 
& $\mathsf{H}_n^D = \mathsf{diag} (\sqrt{\hat{s}_{n,i}^D})$\\ \hline
\end{tabular}
\end{center}
\end{table*}

\subsection{Hyperparameters of optimizers}
\label{a:hyper}
\begin{table*}[htbp]
\begin{center}
\caption{Hyperparameters of the optimizer used in the experiments in Sections \ref{subsec:4.1}, \ref{subsec:4.2}, and \ref{app:big}.}
\label{table:hyper}
\begin{tabular}{c|c|l|l|c|c|c}
\hline

&optimizer
& $\alpha^D$ 
& $\alpha^G$
& $\beta_1^G = \beta_1^D$ 
& $\beta_2^G = \beta_2^D$ 
& $\beta_1$ and $\beta_2$'s reference\\ \hline \hline

&Adam
& $0.0003$
& $0.0001$
& $0.5$
& $0.999$ 
& \citep{radford2016unsupervised} \\

Section \ref{subsec:4.1}
&AdaBelief
& $0.00003$
& $0.0003$
& $0.5$
& $0.999$ 
& \citep{ada} \\

&RMSProp
& $0.00003$
& $0.0001$
& $0$
& $0.99$ 
& \\
\hline

&Adam
& $0.0003$
& $0.0001$
& $0.5$
& $0.999$ 
& \citep{NIPS2017_892c3b1c} \\

Section \ref{subsec:4.2}
&AdaBelief
& $0.00003$
& $0.0005$
& $0.5$
& $0.999$ 
& \citep{ada} \\

&RMSProp
& $0.0001$
& $0.0003$
& $0$
& $0.99$ 
& \\
\hline

Section \ref{app:big}
&Adam
& $0.0004$
& $0.0001$
& $0$
& $0.999$
& \citep{brock2018large} \\

&AdaBelief
& $0.0005$
& $0.00005$
& $0.5$
& $0.999$
& \\
\hline

\end{tabular}
\end{center}
\end{table*}

\subsection{Grid search}
\label{a1}
The combinations of learning rates used in the experiments in Sections \ref{subsec:4.1} and \ref{subsec:4.2} are determined using a grid search. Figure \ref{fig:1} shows the results of the grid search.

\begin{figure}[htbp]
\begin{tabular}{ccc}
\begin{minipage}[t]{0.33\hsize}
\centering
\includegraphics[width=1\textwidth]{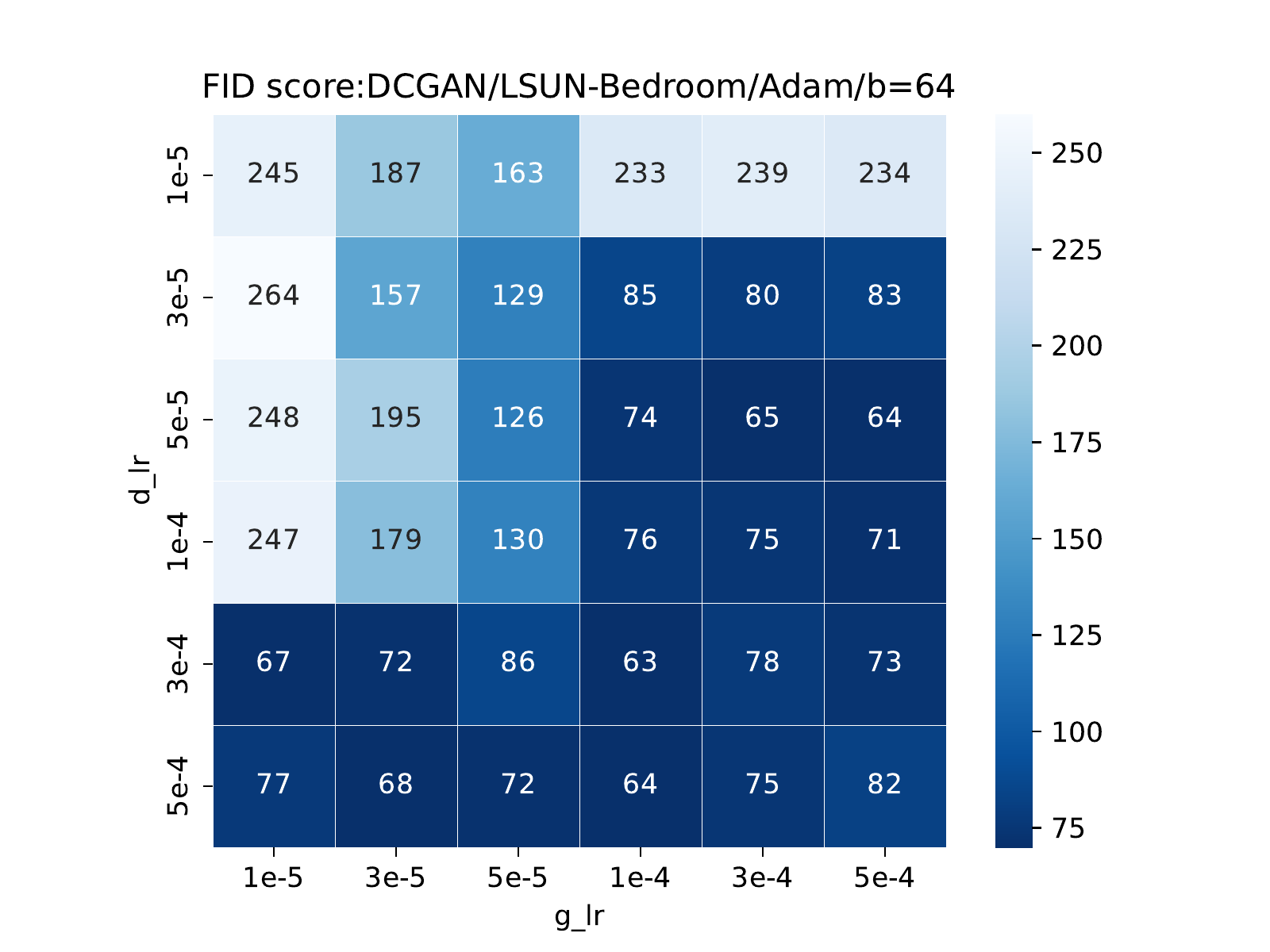}
\end{minipage} 
\begin{minipage}[t]{0.33\hsize}
\centering
\includegraphics[width=1\textwidth]{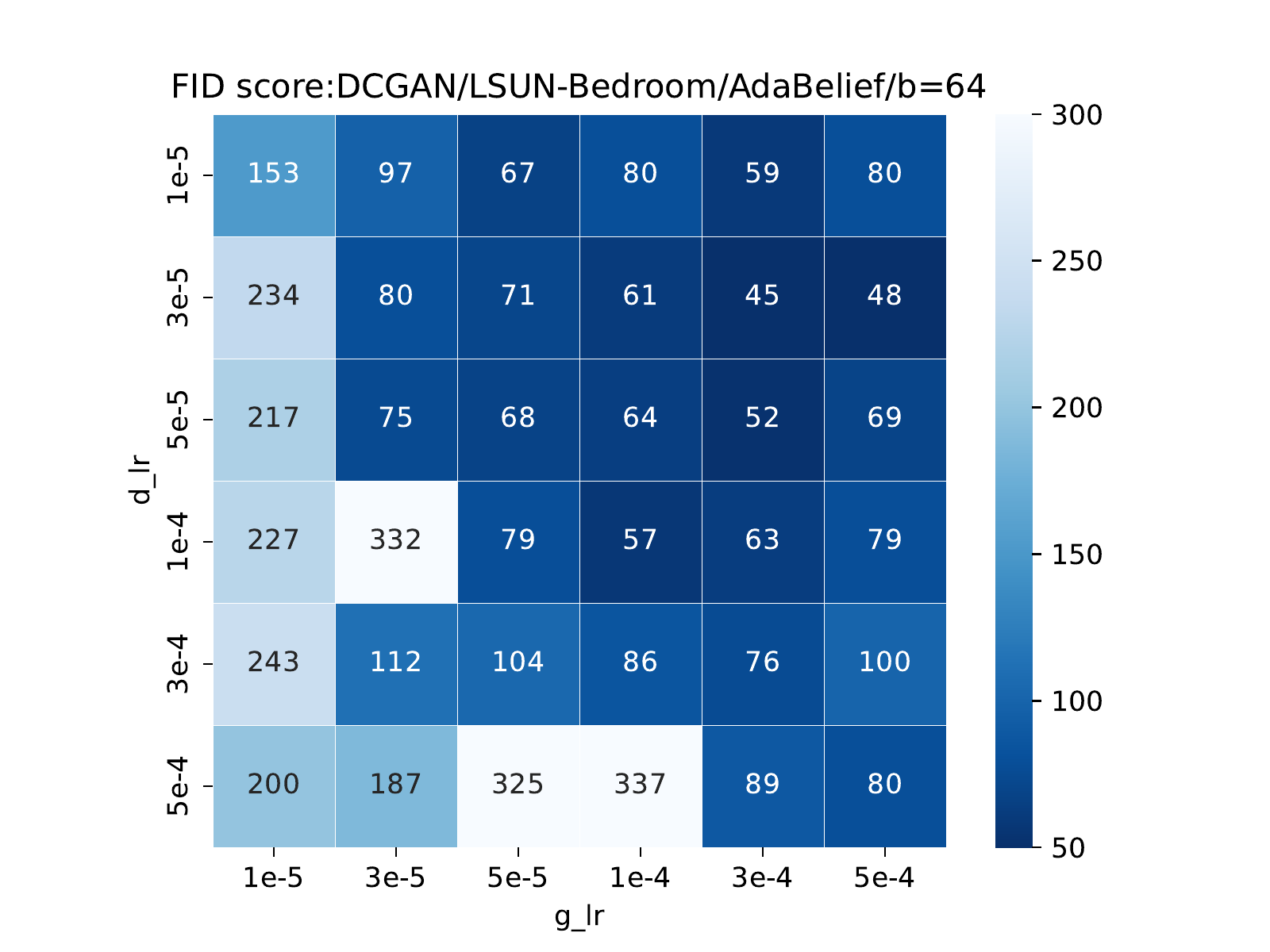}
\end{minipage} 
\begin{minipage}[t]{0.33\hsize}
\centering
\includegraphics[width=1\textwidth]{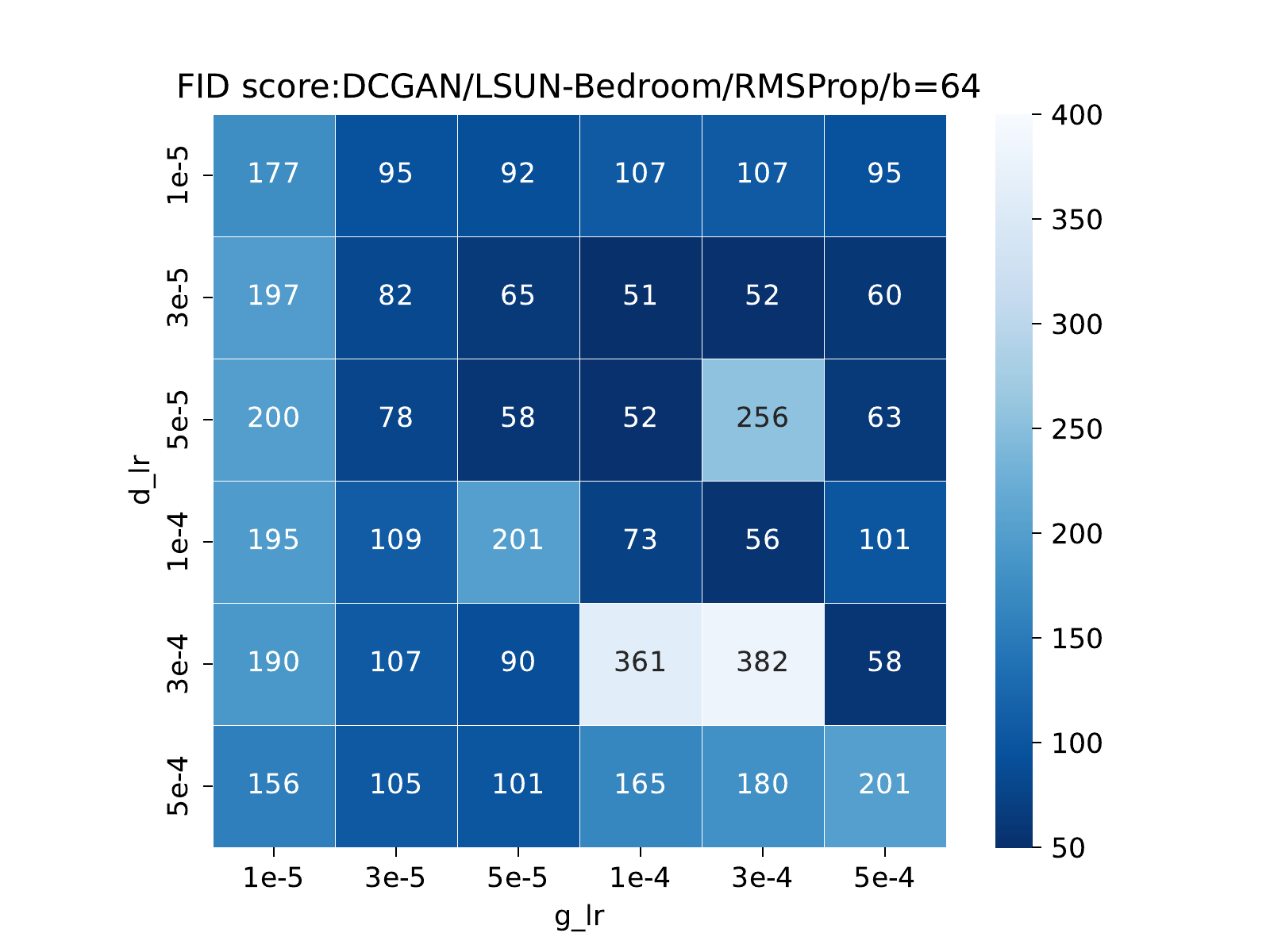}
\end{minipage} \\
\begin{minipage}[t]{0.33\hsize}
\centering
\includegraphics[width=1\textwidth]{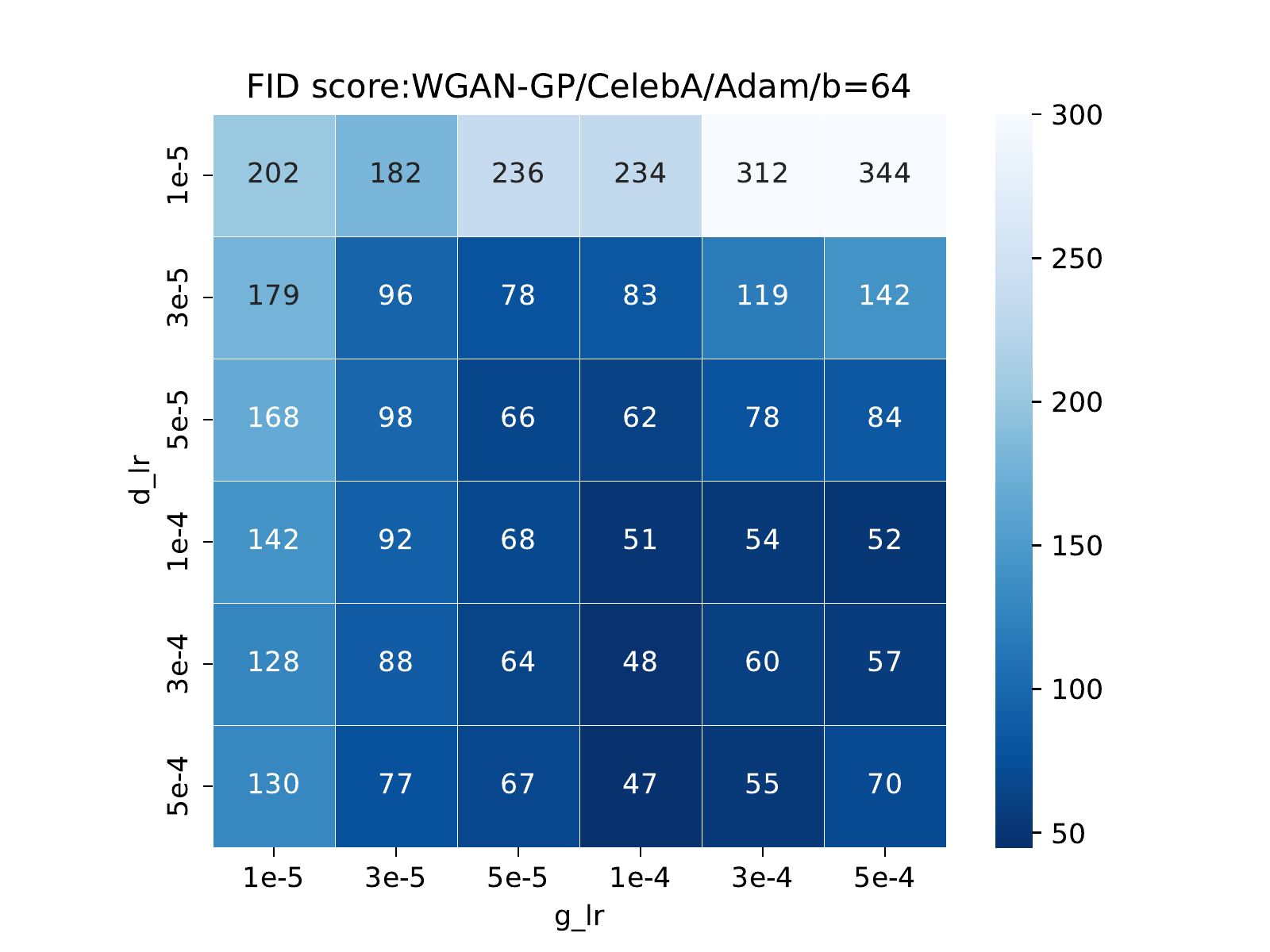}
\end{minipage} 
\begin{minipage}[t]{0.33\hsize}
\centering
\includegraphics[width=1\textwidth]{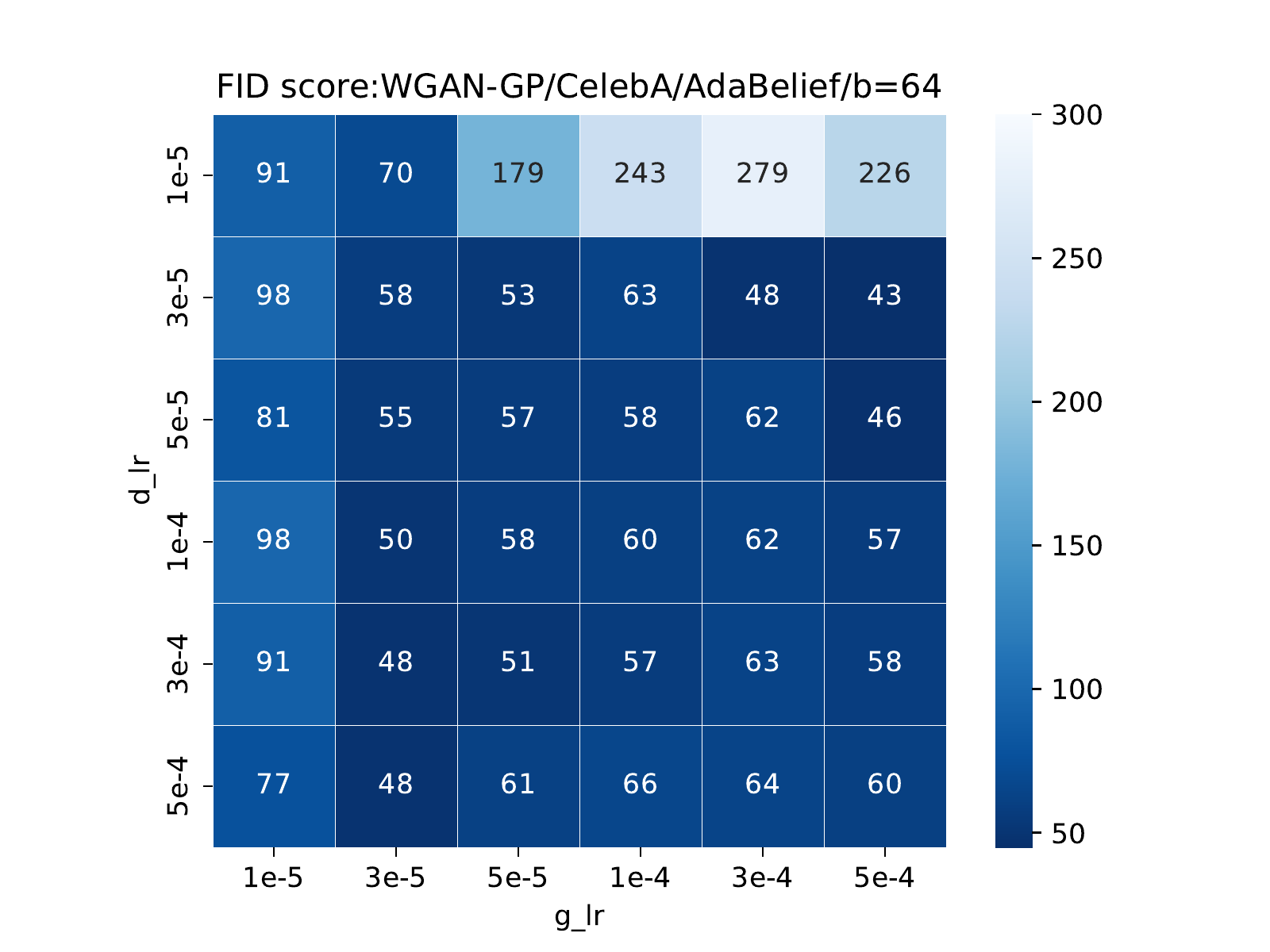}
\end{minipage} 
\begin{minipage}[t]{0.33\hsize}
\centering
\includegraphics[width=1\textwidth]{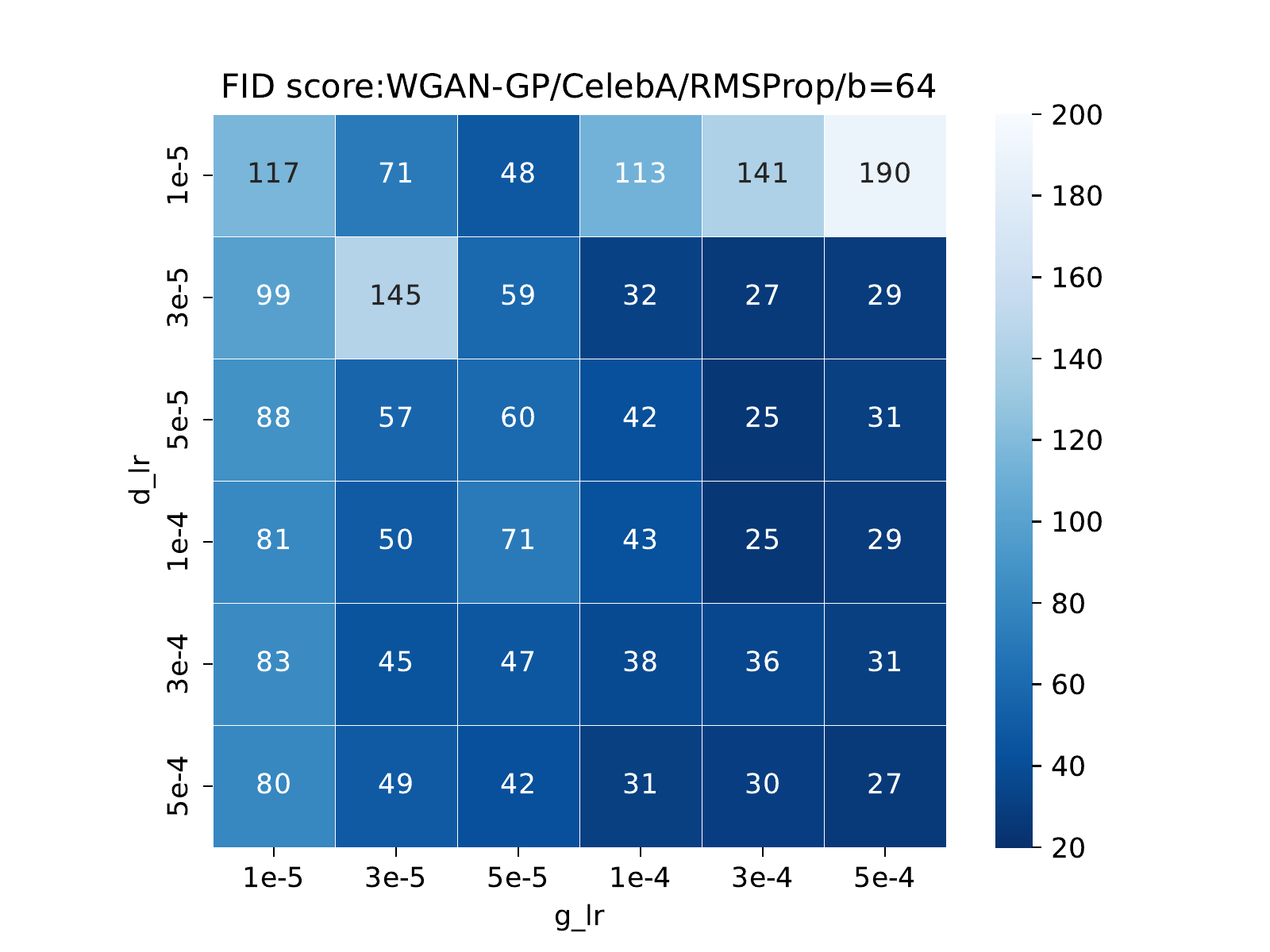}
\end{minipage} 
\end{tabular}
\caption{Analysis of the relationship between combination of learning rates and FID score: discriminator learning rate $\alpha^D$ on the vertical axis and generator learning rate $\alpha^G$ on the horizontal axis. The heatmap colors denote the FID scores; the darker the blue, the lower the FID, meaning that the training of the generator succeeded.}
\label{fig:1}
\end{figure}

\subsection{FID decreases sufficiently}
\label{a3}
We measured the number of steps required to achieve a good FID with different batch sizes. To demonstrate the soundness of the model used in the experiments, Figure \ref{fig:6} shows the decrease in FID. We find that, with DCGAN on the LSUN-Bedroom dataset, the FID score decreases to $41.8$, and with WGAN-GP on the CelebA dataset, it decreases to $24.8$.

\begin{figure}[htbp]
\begin{tabular}{cc}
\begin{minipage}[t]{0.5\hsize}
\centering
\includegraphics[width=1\textwidth]{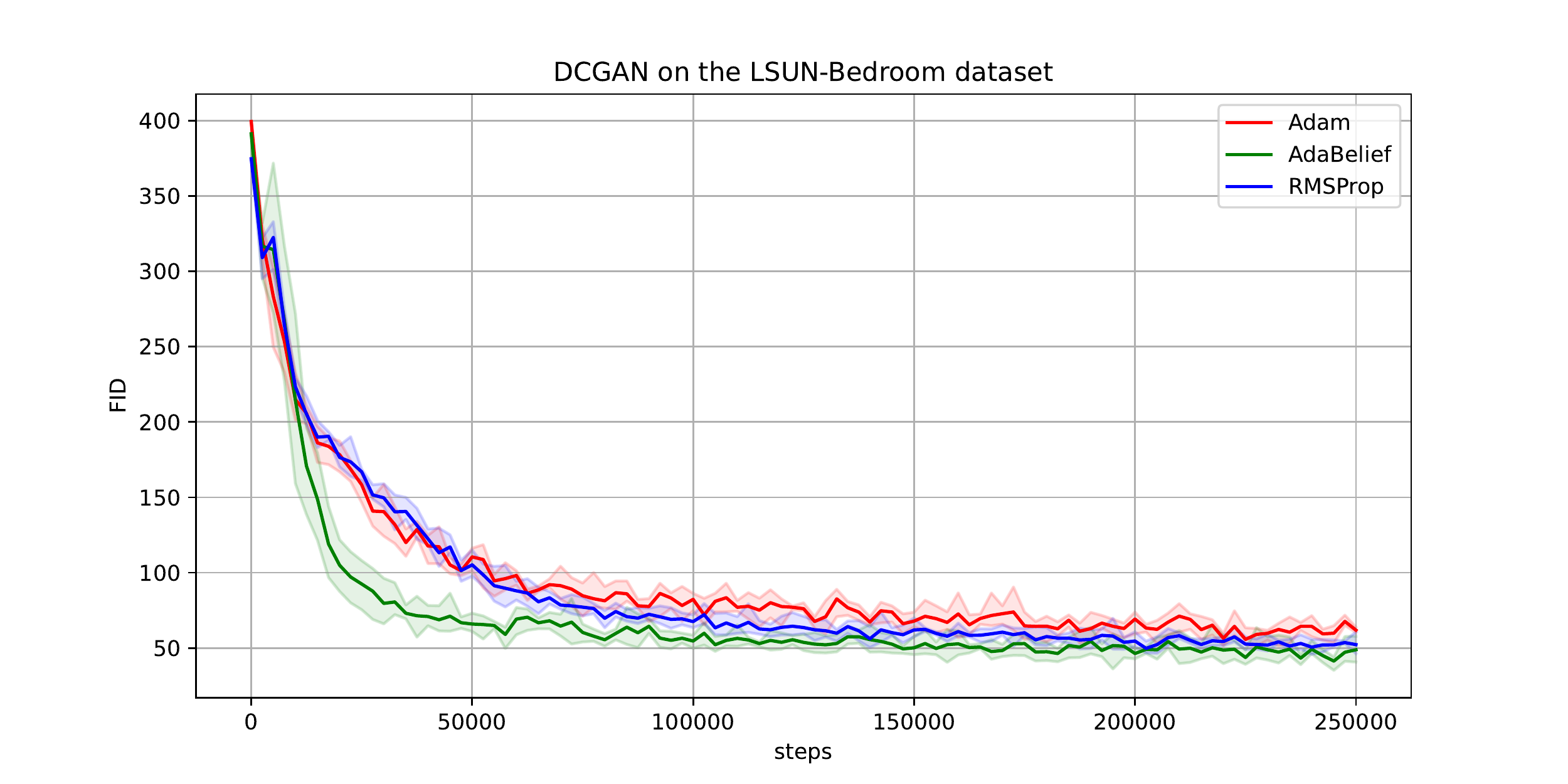}
\end{minipage} 
\begin{minipage}[t]{0.5\hsize}
\centering
\includegraphics[width=1\textwidth]{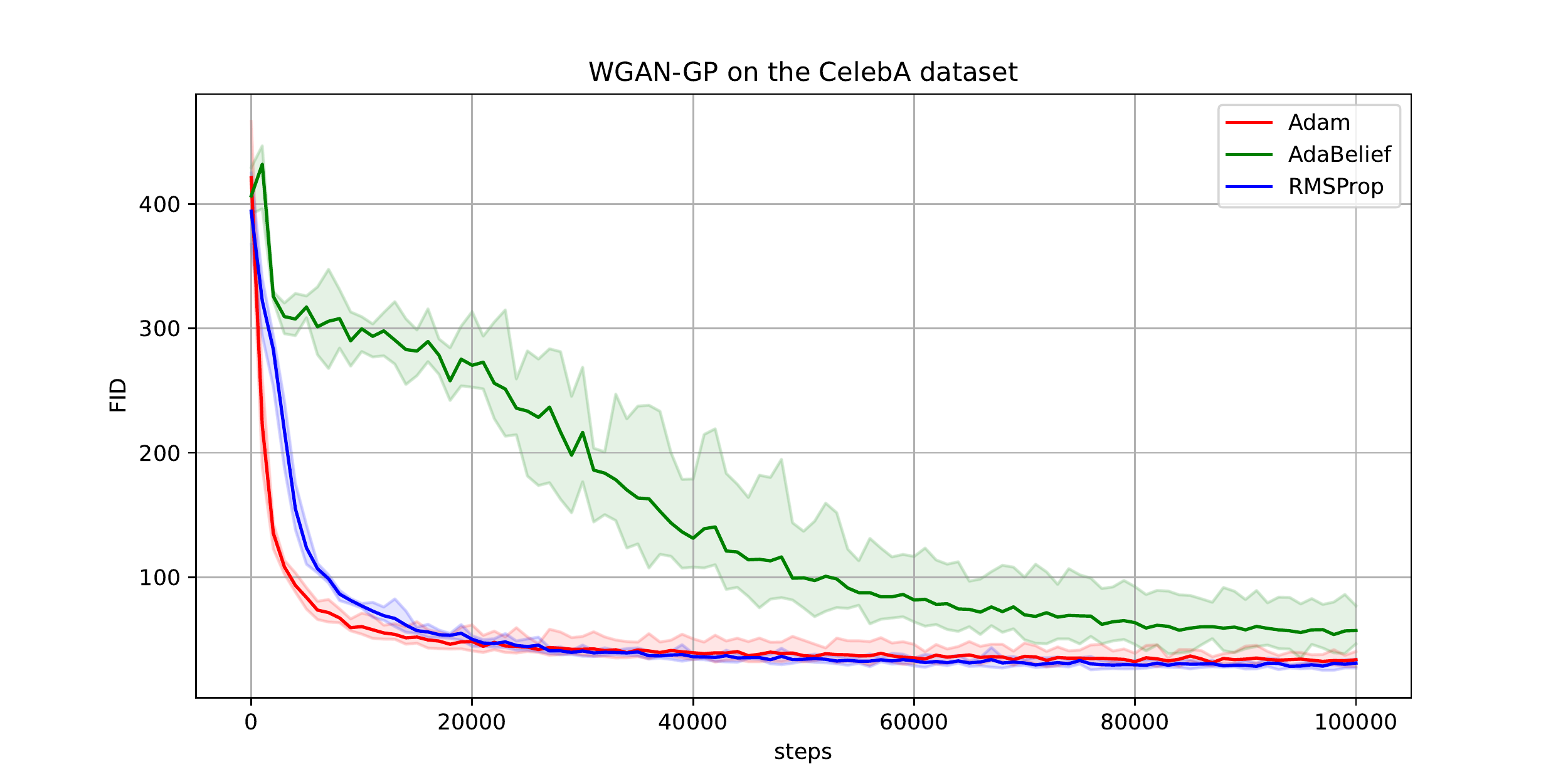}
\end{minipage}
\end{tabular}
\caption{Mean FID (solid line) bounded by the maximum and the minimum over 5 runs (shaded area) for DCGAN on the LSUN-Bedroom dataset and WGAN-GP on the CelebA dataset for three optimizers. For all runs, the batch size is 64 and the learning rate combinations are determined with the same grid search (see Figure \ref{fig:1}) used in Sections \ref{subsec:4.1} and \ref{subsec:4.2}.}
\label{fig:6}
\end{figure}

\subsection{Lemmas}
\begin{lem}\label{lem:1_1}
Suppose that (S1), (S2)(i), and (S3) hold and consider Algorithm \ref{algo:1}. Then, for all $\bm{\theta} \in \mathbb{R}^\Theta$ and all $n\in \mathbb{N}$,
\begin{align*}
\mathbb{E}\left[ \| \bm{\theta}_{n+1} - \bm{\theta} \|_{\mathsf{H}_n^G}^2 \right]
&= 
\mathbb{E}\left[ \| \bm{\theta}_{n} - \bm{\theta} \|_{\mathsf{H}_n^G}^2 \right]
+ \alpha_n^{G^2} \mathbb{E} \left[ \|\bm{d}_n^G \|_{\mathsf{H}_n^G}^2 \right]\\
&\quad + 2 \alpha_n^{G} \left\{
\frac{\beta_1^G}{\tilde{\gamma}_n^G} 
\mathbb{E} \left[ \langle \bm{\theta} - \bm{\theta}_n, \bm{m}_{n-1}^G \rangle \right]
+ 
\frac{\tilde{\beta_1^G}}{\tilde{\gamma}_n^G}
\mathbb{E} \left[ \langle \bm{\theta} - \bm{\theta}_n, \nabla_{\bm{\theta}}
L_G (\bm{\theta}_n, \bm{w}_n) \rangle \right]
\right\}, 
\end{align*}
where $\tilde{\beta_1^G} := 1 - \beta_1^G$ and $\tilde{\gamma}_n^G := 1 - \gamma^{G^{n+1}}$. 
\end{lem}

\begin{proof}
Let $\bm{\theta} \in \mathbb{R}^\Theta$ and $n\in\mathbb{N}$. The definition of $\bm{\theta}_{n+1}$ implies that 
\begin{align*}
\| \bm{\theta}_{n+1} - \bm{\theta} \|_{\mathsf{H}_n^G}^2
= 
\| \bm{\theta}_{n} - \bm{\theta} \|_{\mathsf{H}_n^G}^2
+ 2 \alpha_n^G \langle \bm{\theta}_{n} - \bm{\theta}, \bm{d}_n^G \rangle_{\mathsf{H}_n^G}
+ \alpha_n^{G^2} \|\bm{d}_n^G \|_{\mathsf{H}_n^G}^2.
\end{align*}
Moreover, the definitions of $\bm{d}_n^G$, $\bm{m}_n^G$, and $\hat{\bm{m}}_n^G$ ensure that 
\begin{align*}
\left\langle \bm{\theta}_n - \bm{\theta}, \bm{d}_n^G \right\rangle_{\mathsf{H}_n^G}
&=
\frac{1}{{\tilde{\gamma}}_n^G}
\langle \bm{\theta} - \bm{\theta}_n, \bm{m}_n^G \rangle\\
&=
\frac{\beta_1^G}{{\tilde{\gamma}}_n^G} 
\langle \bm{\theta} - \bm{\theta}_n, \bm{m}_{n-1}^G \rangle
+
\frac{\tilde{\beta_1^G}}{{\tilde{\gamma}}_n^G} 
\langle \bm{\theta} - \bm{\theta}_n, \nabla L_{G,\mathcal{S}_n}(\bm{\theta}_n) \rangle.
\end{align*}
Hence, 
\begin{align}\label{ineq:004}
\begin{split}
\left\|\bm{\theta}_{n+1} - \bm{\theta} \right\|_{\mathsf{H}_n^G}^2
&=
\left\| \bm{\theta}_n -\bm{\theta} \right\|_{\mathsf{H}_n^G}^2
+ \alpha_n^{G^2} \| \bm{d}_n^G \|_{\mathsf{H}_n^G}^2\\
&\quad + 2 \alpha_n^G \left\{
\frac{\beta_1^G}{{\tilde{\gamma}}_n^G} 
\langle \bm{\theta} - \bm{\theta}_n, \bm{m}_{n-1}^G \rangle
+
\frac{\tilde{\beta_1^G}}{{\tilde{\gamma}}_n^G} 
\langle \bm{\theta} - \bm{\theta}_n, \nabla L_{G,\mathcal{S}_n}(\bm{\theta}_n) \rangle 
\right\}.
\end{split}
\end{align}
Conditions (S2)(i) and (S3) guarantee that
\begin{align*}
\mathbb{E}\left[ \mathbb{E} \left[ \langle \bm{\theta} - \bm{\theta}_n, \nabla L_{G,\mathcal{S}_n}(\bm{\theta}_n) \rangle \Big| \bm{\theta}_n \right] \right]
&=
\mathbb{E} \left[ \left\langle \bm{\theta} - \bm{\theta}_n, 
\mathbb{E} \left[\nabla L_{G,\mathcal{S}_n}(\bm{\theta}_n) \Big| \bm{\theta}_n \right] \right\rangle \right]\\
&=
\mathbb{E} \left[ \langle \bm{\theta} - \bm{\theta}_n, 
\nabla_{\bm{\theta}} L_G (\bm{\theta}_n, \bm{w}_n) \rangle \right].
\end{align*}
The lemma follows by taking the expectation with respect to $\xi_n^G$ on both sides of \eqref{ineq:004}.
\end{proof}

A discussion similar to the one proving Lemma \ref{lem:1_1} leads to the following lemma.

\begin{lem}\label{lem:1_2}
Suppose that (S1), (S2)(i), and (S3) hold and consider Algorithm \ref{algo:1}. Then, for all $\bm{w} \in \mathbb{R}^W$ and all $n\in \mathbb{N}$,
\begin{align*}
\mathbb{E}\left[ \| \bm{w}_{n+1} - \bm{w} \|_{\mathsf{H}_n^D}^2 \right]
&= 
\mathbb{E}\left[ \| \bm{w}_{n} - \bm{w} \|_{\mathsf{H}_n^D}^2 \right]
+ \alpha_n^{D^2} \mathbb{E} \left[ \|\bm{d}_n^D \|_{\mathsf{H}_n^D}^2 \right]\\
&\quad + 2 \alpha_n^{D} \left\{
\frac{\beta_1^D}{\tilde{\gamma}_n^D} 
\mathbb{E} \left[ \langle \bm{w} - \bm{w}_n, \bm{m}_{n-1}^D \rangle \right]
+ 
\frac{\tilde{\beta_1^D}}{\tilde{\gamma}_n^D}
\mathbb{E} \left[ \langle \bm{w} - \bm{w}_n, \nabla_{\bm{w}}
L_D (\bm{\theta}_n, \bm{w}_n) \rangle \right]
\right\}, 
\end{align*}
where $\tilde{\beta_1^D} := 1 - \beta_1^D$ and $\tilde{\gamma}_n^D := 1 - \gamma^{D^{n+1}}$. 
\end{lem}

\begin{lem}\label{lem:2_1}
Algorithm \ref{algo:1} satisfies that, under (S2)(i), (ii) and (C2), for all $n\in\mathbb{N}$,
\begin{align*}
\mathbb{E}\left[ \|\bm{m}_n^G\|^2 \right] 
\leq 
\frac{\sigma_G^2}{b} + M_G^2.
\end{align*}
Under (A1) and (C2), for all $k\in\mathbb{N}$,
\begin{align*}
\mathbb{E}\left[ \|\bm{d}_n^G\|_{\mathsf{H}_n^G}^2 \right] 
\leq 
\frac{1}{(1-{\gamma^G})^2 h_{0,*}^G} \left( \frac{\sigma_G^2}{b} + M_G^2 \right),
\end{align*}
where $h_{0,*}^G := \min_{i\in [\Theta]} h_{0,i}^G$.
\end{lem}

\begin{proof}
Let $n\in\mathbb{N}$. From (S2)(i), we have 
\begin{align*}
\mathbb{E} \left[\| \nabla L_{G,\mathcal{S}_n} (\bm{\theta}_{n}) \|^2
\big| \bm{\theta}_n
\right]
&=
\mathbb{E} \left[\| \nabla L_{G,\mathcal{S}_n} (\bm{\theta}_{n}) 
- \nabla L_G (\bm{\theta}_{n},\bm{w}_n) + \nabla L_G (\bm{\theta}_{n},\bm{w}_n)\|^2
\big| \bm{\theta}_n
\right]\\
&=
\mathbb{E} \left[\| \nabla L_{G,\mathcal{S}_n} (\bm{\theta}_{n}) 
- \nabla L_G (\bm{\theta}_{n},\bm{w}_n) \|^2 \big| \bm{\theta}_n
\right]
+ 
\mathbb{E} \left[\| \nabla L_G (\bm{\theta}_{n},\bm{w}_n) \|^2 \big| \bm{\theta}_n
\right]\\
&\quad + 2 
\mathbb{E} \left[ \langle 
\nabla L_{G,\mathcal{S}_n} (\bm{\theta}_{n}) 
- \nabla L_G (\bm{\theta}_{n},\bm{w}_n), \nabla L_G (\bm{\theta}_{n},\bm{w}_n) \rangle
\Big| \bm{\theta}_n \right]\\
&= 
\mathbb{E} \left[\| \nabla L_{G,\mathcal{S}_n} (\bm{\theta}_{n}) 
- \nabla L_G (\bm{\theta}_{n},\bm{w}_n) \|^2 \big| \bm{\theta}_n
\right]
+ 
\| \nabla L_G (\bm{\theta}_{n},\bm{w}_n) \|^2,
\end{align*} 
which, together with (S2)(ii) and (C2), implies that 
\begin{align}\label{A3}
\mathbb{E} \left[\| \nabla L_{G,\mathcal{S}_n} (\bm{\theta}_{n}) \|^2
\right]
\leq 
\frac{\sigma_G^2}{b} + M_G^2.
\end{align}
The convexity of $\|\cdot\|^2$, together with the definition of $\bm{m}_n^G$ and \eqref{A3}, guarantees that, for all $n\in\mathbb{N}$,
\begin{align*}
\mathbb{E}\left[ \|\bm{m}_n^G\|^2 \right]
&\leq \beta_1^G \mathbb{E}\left[ \|\bm{m}_{n-1}^G \|^2 \right] + 
(1-\beta_1^G) \mathbb{E}\left[ \|\nabla L_{G,\mathcal{S}_n} (\bm{\theta}_{n}) \|^2 \right]\\
&\leq 
\beta_1^G \mathbb{E} \left[ \|\bm{m}_{n-1}^G \|^2 \right] + (1-\beta_1^G) 
\left( \frac{\sigma_G^2}{b} + M_G^2 \right).
\end{align*}
Induction thus ensures that, for all $n\in\mathbb{N}$,
\begin{align}\label{induction}
\mathbb{E} \left[\|\bm{m}_n^G \|^2 \right] \leq 
\max \left\{ \|\bm{m}_{-1}^G\|^2, \frac{\sigma_G^2}{b} + M_G^2 \right\} 
= \frac{\sigma_G^2}{b} + M_G^2,
\end{align}
where $\bm{m}_{-1}^G = \bm{0}$ is used. For $n\in\mathbb{N}$, $\mathsf{H}_n^G \in \mathbb{S}_{++}^\Theta$ guarantees the existence of a unique matrix $\overline{\mathsf{H}}_n^G \in \mathbb{S}_{++}^\Theta$ such that $\mathsf{H}_n^G = \overline{\mathsf{H}}_n^{G^2}$ \citep[Theorem 7.2.6]{horn}. We have that, for all $\bm{x}\in\mathbb{R}^\Theta$, $\|\bm{x}\|_{\mathsf{H}_n^G}^2 = \| \overline{\mathsf{H}}_n^G \bm{x} \|^2$. Accordingly, the definitions of $\bm{d}_n^G$ and $\hat{\bm{m}}_n^G$ imply that, for all $n\in\mathbb{N}$, 
\begin{align*}
\mathbb{E} \left[ \| \bm{d}_n^G \|_{\mathsf{H}_n^G}^2 \right]
&= 
\mathbb{E} \left[ \left\| \overline{\mathsf{H}}_n^{G^{-1}} \mathsf{H}_n^G \bm{d}_n^G \right\|^2 \right]
\leq 
\frac{1}{{\tilde{\gamma}}_n^{G^2}} \mathbb{E} \left[ \left\| \overline{\mathsf{H}}_n^{G^{-1}} \right\|^2 \|\bm{m}_n^G \|^2 \right]\\
&\leq 
\frac{1}{(1 - \gamma^G)^2} \mathbb{E} \left[ \left\| \overline{\mathsf{H}}_n^{G^{-1}} \right\|^2 \|\bm{m}_n^G \|^2 \right],
\end{align*}
where 
\begin{align*}
\left\| \overline{\mathsf{H}}_n^{G^{-1}} \right\| 
= \left\| \mathsf{diag}\left(h_{n,i}^{G^{-\frac{1}{2}}} \right) \right\| 
= {\max_{i\in [\Theta]} h_{n,i}^{G^{-\frac{1}{2}}}} 
\end{align*}
and ${\tilde{\gamma}}_n^G := 1 - {\gamma}^{G^{n+1}} \geq 1 - {\gamma}^G$. Moreover, (A1) ensures that, for all $n \in \mathbb{N}$, 
\begin{align*}
h_{n,i}^G \geq h_{0,i}^G \geq h_{0,*}^G := \min_{i\in [\Theta]} h_{0,i}^G.
\end{align*}
Hence, (\ref{induction}) implies that, for all $k\in \mathbb{N}$,
\begin{align*}
\mathbb{E} \left[\| \bm{d}_n^G \|_{\mathsf{H}_n^G}^2 \right] \leq 
\frac{1}{(1-{\gamma}^G)^2 h_{0,*}^G} \left( \frac{\sigma_G^2}{b} + M_G^2 \right),
\end{align*}
completing the proof.
\end{proof}

A discussion similar to the one proving Lemma \ref{lem:2_1} leads to the following lemma.

\begin{lem}\label{lem:2_2}
Algorithm \ref{algo:1} satisfies that, under (S2)(i), (ii) and (C2), for all $n\in\mathbb{N}$,
\begin{align*}
\mathbb{E}\left[ \|\bm{m}_n^D\|^2 \right] 
\leq 
\frac{\sigma_D^2}{b} + M_D^2.
\end{align*}
Under (A1) and (C2), for all $k\in\mathbb{N}$,
\begin{align*}
\mathbb{E}\left[ \|\bm{d}_n^D\|_{\mathsf{H}_n^D}^2 \right] 
\leq 
\frac{1}{(1-{\gamma^D})^2 h_{0,*}^D} \left( \frac{\sigma_D^2}{b} + M_D^2 \right),
\end{align*}
where $h_{0,*}^D := \min_{i\in [W]} h_{0,i}^D$.
\end{lem}

Lemmas \ref{lem:1_1} and \ref{lem:2_1} lead to the following:

\begin{lem}\label{lem:3_1}
Suppose that (S1)--(S3), (A1), and (C2)--(C3) hold and define $X_n^G (\bm{\theta}) := \mathbb{E}[ \| \bm{\theta}_{n} - \bm{\theta} \|_{\mathsf{H}_n^G}^2 ]$ for all $n\in\mathbb{N}$ and all $\bm{\theta} \in \mathbb{R}^\Theta$. Then, for all $\bm{\theta} \in \mathbb{R}^\Theta$ and all $n\in \mathbb{N}$,
\begin{align*}
X_{n+1}^G(\bm{\theta})
&\leq 
X_n^G(\bm{\theta})
+
\mathrm{Dist}(\bm{\theta})
\mathbb{E}\left[ 
\sum_{i\in [\Theta]} (h_{n+1,i}^G - h_{n,i}^G)
\right]
+ \frac{\alpha_n^{G^2}}{(1-{\gamma^G})^2 h_{0,*}^G} \left( \frac{\sigma_G^2}{b} + M_G^2 \right)\\
&\quad + 2 \alpha_n^{G} \left\{
\frac{\beta_1^G}{\tilde{\gamma}_n^G} 
\sqrt{\Theta \mathrm{Dist}(\bm{\theta})
\left( \frac{\sigma_G^2}{b} + M_G^2 \right)}
+ 
\frac{\tilde{\beta_1^G}}{\tilde{\gamma}_n^G}
\mathbb{E} \left[ \langle \bm{\theta} - \bm{\theta}_n, \nabla_{\bm{\theta}}
L_G (\bm{\theta}_n, \bm{w}_n) \rangle \right]
\right\}. 
\end{align*}
\end{lem}

\begin{proof}
Lemma \ref{lem:2_1} and Jensen's inequality guarantee that
\begin{align*}
\mathbb{E}\left[ \|\bm{m}_n^G\| \right] 
\leq 
\sqrt{\frac{\sigma_G^2}{b} + M_G^2}.
\end{align*}
Condition (C3) implies that, for all $\bm{\theta} \in \mathbb{R}^\Theta$, 
\begin{align*}
\| \bm{\theta}_n - \bm{\theta}\|^2
= 
\sum_{i\in [\Theta]} (\theta_{n,i} - \theta_i)^2
\leq
\Theta \mathrm{Dist}(\bm{\theta}).
\end{align*}
The Cauchy-Schwarz inequality thus ensures that 
\begin{align}\label{CS}
\mathbb{E} \left[ \langle \bm{\theta} - \bm{\theta}_n, \bm{m}_{n-1}^G \rangle \right]
\leq
\mathbb{E} \left[ \| \bm{\theta} - \bm{\theta}_n\| \|\bm{m}_{n-1}^G \| \right]
\leq
\sqrt{\Theta \mathrm{Dist}(\bm{\theta})
\left( \frac{\sigma_G^2}{b} + M_G^2 \right)}.
\end{align}
We define $X_n^G (\bm{\theta}) := \mathbb{E}[ \| \bm{\theta}_{n} - \bm{\theta} \|_{\mathsf{H}_n^G}^2 ]$ for all $n\in\mathbb{N}$ and all $\bm{\theta} \in \mathbb{R}^\Theta$. Then, we have 
\begin{align*}
X_{n+1}^G (\bm{\theta}) - \mathbb{E}\left[ \| \bm{\theta}_{n+1} - \bm{\theta} \|_{\mathsf{H}_n^G}^2 \right]
= 
\mathbb{E}\left[ 
\sum_{i\in [\Theta]} (h_{n+1,i}^G - h_{n,i}^G)(\theta_{n+1,i} - \theta_i)^2
\right],
\end{align*}
which, together with (C3), implies that
\begin{align*}
X_{n+1}^G (\bm{\theta}) - \mathbb{E}\left[ \| \bm{\theta}_{n+1} - \bm{\theta} \|_{\mathsf{H}_n^G}^2 \right]
\leq
\mathrm{Dist}(\bm{\theta})
\mathbb{E}\left[ 
\sum_{i\in [\Theta]} (h_{n+1,i}^G - h_{n,i}^G)
\right].
\end{align*}
Hence, Lemmas \ref{lem:1_1} and \ref{lem:2_1} lead to the assertion in Lemma \ref{lem:3_1}.
\end{proof}

A discussion similar to the one proving Lemma \ref{lem:3_1}, together with Lemmas \ref{lem:1_2} and \ref{lem:2_2}, leads to the following lemma.

\begin{lem}\label{lem:3_2}
Suppose that (S1)--(S3), (A1), and (C2)--(C3) hold and define $X_n^D (\bm{w}) := \mathbb{E}[ \| \bm{w}_{n} - \bm{w} \|_{\mathsf{H}_n^D}^2 ]$ for all $n\in\mathbb{N}$ and all $\bm{w} \in \mathbb{R}^W$. Then, for all $\bm{w} \in \mathbb{R}^W$ and all $n\in \mathbb{N}$,
\begin{align*}
X_{n+1}^D(\bm{w})
&\leq 
X_n^D(\bm{w})
+
\mathrm{Dist}(\bm{w})
\mathbb{E}\left[ 
\sum_{i\in [W]} (h_{n+1,i}^D - h_{n,i}^D)
\right]
+ \frac{\alpha_n^{D^2}}{(1-{\gamma^D})^2 h_{0,*}^D} \left( \frac{\sigma_D^2}{b} + M_D^2 \right)\\
&\quad + 2 \alpha_n^{D} \left\{
\frac{\beta_1^D}{\tilde{\gamma}_n^D} 
\sqrt{W \mathrm{Dist}(\bm{w})
\left( \frac{\sigma_D^2}{b} + M_D^2 \right)}
+ 
\frac{\tilde{\beta_1^D}}{\tilde{\gamma}_n^D}
\mathbb{E} \left[ \langle \bm{w} - \bm{w}_n, \nabla_{\bm{w}}
L_D (\bm{\theta}_n, \bm{w}_n) \rangle \right]
\right\}. 
\end{align*}
\end{lem}

\subsection{Proof of Theorem \ref{thm:1}(i)}
\label{a5}
The following is a convergence analysis of Algorithm \ref{algo:1}.

\begin{thm}
Suppose that Assumptions \ref{assum:0}, \ref{assum:1}, and \ref{assum:c} hold and consider the sequence $((\bm{\theta}_{n}, \bm{w}_{n}))_{n\in \mathbb{N}}$ generated by Algorithm \ref{algo:1}. Then, the following hold: For all $\bm{\theta} \in \mathbb{R}^\Theta$ and all $\bm{w} \in \mathbb{R}^W$,
\begin{align*}
&\liminf_{n \to + \infty}
\mathbb{E} \left[ \langle \bm{\theta}_n - \bm{\theta}, \nabla_{\bm{\theta}}
L_G (\bm{\theta}_n, \bm{w}_n) \rangle \right]
\leq
\frac{\alpha^{G} (\sigma_G^2 b^{-1} + M_G^2)}{2\tilde{\beta_1^G} \tilde{\gamma}^{G^2} h_{0,*}^G} 
+ \sqrt{\Theta \mathrm{Dist}(\bm{\theta})
\left( \frac{\sigma_G^2}{b} + M_G^2 \right)}
\frac{\beta_1^G}{\tilde{\beta_1^G}},\\
&\liminf_{n\to + \infty} \mathbb{E}\left[ \langle \bm{w}_n - \bm{w},\nabla_{\bm{w}} L_D (\bm{\theta}_n, \bm{w}_n) \rangle \right]
\leq 
\frac{\alpha^{D} (\sigma_D^2 b^{-1} + M_D^2)}{2\tilde{\beta_1^D} \tilde{\gamma}^{D^2} h_{0,*}^D}
+ \sqrt{W \mathrm{Dist}(\bm{w})
\left( \frac{\sigma_D^2}{b} + M_D^2 \right)}
\frac{\beta_1^D}{\tilde{\beta_1^D}}, 
\end{align*}
where $\tilde{\beta_1^G} := 1 - \beta_1^G$, $\tilde{\beta_1^D} := 1 - \beta_1^D$, $\tilde{\gamma}^G := 1 - \gamma^G$, $\tilde{\gamma}^D := 1 - \gamma^D$, $h_{0,*}^G := \min_{i\in [\Theta]}h_{0,i}^G$, and $h_{0,*}^D := \min_{j\in [W]}h_{0,j}^D$. Furthermore, there exist accumulation points $(\bm{\theta}^*, \bm{w}^*)$ and $(\bm{\theta}_*, \bm{w}_*)$ of $((\bm{\theta}_{n}, \bm{w}_{n}))_{n\in \mathbb{N}}$ such that
\begin{align*}
&\mathbb{E}\left[ \|\nabla_{\bm{\theta}} L_G (\bm{\theta}^*, \bm{w}^*) \|^2 \right]
\leq 
\frac{\alpha^{G} (\sigma_G^2 b^{-1} + M_G^2)}{2\tilde{\beta_1^G} \tilde{\gamma}^{G^2} h_{0,*}^G}
+ \sqrt{\Theta \mathrm{Dist}(\tilde{\bm{\theta}})
\left( \frac{\sigma_G^2}{b} + M_G^2 \right)}
\frac{\beta_1^G}{\tilde{\beta_1^G}}, \\
&\mathbb{E}\left[ \|\nabla_{\bm{w}} L_D (\bm{\theta}_*, \bm{w}_*) \|^2 \right]
\leq 
\frac{\alpha^{D} (\sigma_D^2 b^{-1} + M_D^2)}{2\tilde{\beta_1^D} \tilde{\gamma}^{D^2} h_{0,*}^D} 
+ \sqrt{W \mathrm{Dist}(\tilde{\bm{w}})
\left( \frac{\sigma_D^2}{b} + M_D^2 \right)}
\frac{\beta_1^D}{\tilde{\beta_1^D}},
\end{align*}
where $\tilde{\bm{\theta}} := \bm{\theta}^* - \nabla_{\bm{\theta}} L_G (\bm{\theta}^*, \bm{w}^*)$ and $\tilde{\bm{w}} := \bm{w}_* - \nabla_{\bm{w}} L_D (\bm{\theta}_*, \bm{w}_*)$.
\end{thm}

\begin{proof}
Let us assume (C1), i.e., $\alpha_n^G := \alpha^G$ for all $n\in\mathbb{N}$. Then, Lemma \ref{lem:3_1} ensures that, for all $\bm{\theta} \in \mathbb{R}^\Theta$ and all $n\in \mathbb{N}$,
\begin{align*}
X_{n+1}^G(\bm{\theta})
&\leq 
X_n^G(\bm{\theta})
+
\mathrm{Dist}(\bm{\theta})
\mathbb{E}\left[ 
\sum_{i\in [\Theta]} (h_{n+1,i}^G - h_{n,i}^G)
\right]
+ \frac{\alpha^{G^2}}{(1-{\gamma^G})^2 h_{0,*}^G} \left( \frac{\sigma_G^2}{b} + M_G^2 \right)\\
&\quad + 2 \alpha^{G} \left\{
\frac{\beta_1^G}{\tilde{\gamma}_n^G} 
\sqrt{\Theta \mathrm{Dist}(\bm{\theta})
\left( \frac{\sigma_G^2}{b} + M_G^2 \right)}
+ 
\frac{\tilde{\beta_1^G}}{\tilde{\gamma}_n^G}
\mathbb{E} \left[ \langle \bm{\theta} - \bm{\theta}_n, \nabla_{\bm{\theta}}
L_G (\bm{\theta}_n, \bm{w}_n) \rangle \right]
\right\}. 
\end{align*}
Since we have that $\tilde{\gamma}_n^G = 1 - \gamma^{G^{n+1}} \leq 1$, $\gamma^{G^{n+1}} (X_{n+1}^G(\bm{\theta}) - X_n^G(\bm{\theta})) \leq \gamma^{G^{n+1}} X_{n+1}^G(\bm{\theta})$, and $h_{n+1,i}^G \geq h_{n,i}^G$ (by (A1)) for all $n\in \mathbb{N}$, we also have that
\begin{align}\label{ineq:0}
\begin{split}
X_{n+1}^G(\bm{\theta})
&\leq 
X_n^G(\bm{\theta})
+
\gamma^{G^{n+1}} X_{n+1}^G(\bm{\theta})
+
\mathrm{Dist}(\bm{\theta})
\mathbb{E}\left[ 
\sum_{i\in [\Theta]} (h_{n+1,i}^G - h_{n,i}^G)
\right]\\
&\quad + \frac{\alpha^{G^2}}{(1-{\gamma^G})^2 h_{0,*}^G} \left( \frac{\sigma_G^2}{b} + M_G^2 \right)\\
&\quad + 2 \alpha^{G} \left\{
\beta_1^G 
\sqrt{\Theta \mathrm{Dist}(\bm{\theta})
\left( \frac{\sigma_G^2}{b} + M_G^2 \right)}
+ 
\tilde{\beta_1^G}
\mathbb{E} \left[ \langle \bm{\theta} - \bm{\theta}_n, \nabla_{\bm{\theta}}
L_G (\bm{\theta}_n, \bm{w}_n) \rangle \right]
\right\}. 
\end{split} 
\end{align}
Let us show that, for all $\bm{\theta} \in \mathbb{R}^\Theta$ and all $\epsilon > 0$,
\begin{align}\label{liminf_1}
\begin{split}
&\liminf_{n \to + \infty}
\mathbb{E} \left[ \langle \bm{\theta}_n - \bm{\theta}, \nabla_{\bm{\theta}}
L_G (\bm{\theta}_n, \bm{w}_n) \rangle \right]\\
&\leq
\frac{\alpha^{G}}{2\tilde{\beta_1^G} (1-{\gamma^G})^2 h_{0,*}^G} \left( \frac{\sigma_G^2}{b} + M_G^2 \right)
+ \sqrt{\Theta \mathrm{Dist}(\bm{\theta})
\left( \frac{\sigma_G^2}{b} + M_G^2 \right)}
\frac{\beta_1^G}{\tilde{\beta_1^G}} + \epsilon.
\end{split}
\end{align}
If (\ref{liminf_1}) does not hold, then there exist $\hat{\bm{\theta}} \in \mathbb{R}^\Theta$ and $\epsilon_0 > 0$ such that
\begin{align*}
&\liminf_{n \to + \infty}
\mathbb{E} \left[ \langle \bm{\theta}_n - \hat{\bm{\theta}}, \nabla_{\bm{\theta}}
L_G (\bm{\theta}_n, \bm{w}_n) \rangle \right]\\
&>
\frac{\alpha^{G}}{2\tilde{\beta_1^G} (1-{\gamma^G})^2 h_{0,*}^G} \left( \frac{\sigma_G^2}{b} + M_G^2 \right)
+ \sqrt{\Theta \mathrm{Dist}(\hat{\bm{\theta}})
\left( \frac{\sigma_G^2}{b} + M_G^2 \right)}
\frac{\beta_1^G}{\tilde{\beta_1^G}} + \epsilon_0.
\end{align*}
Then, there exists $n_0 \in \mathbb{N}$ such that, for all $n \geq n_0$,
\begin{align*}
&
\mathbb{E} \left[ \langle \bm{\theta}_n - \hat{\bm{\theta}}, \nabla_{\bm{\theta}}
L_G (\bm{\theta}_n, \bm{w}_n) \rangle \right]\\
&>
\frac{\alpha^{G}}{2\tilde{\beta_1^G} (1-{\gamma^G})^2 h_{0,*}^G} \left( \frac{\sigma_G^2}{b} + M_G^2 \right)
+ \sqrt{\Theta \mathrm{Dist}(\hat{\bm{\theta}})
\left( \frac{\sigma_G^2}{b} + M_G^2 \right)}
\frac{\beta_1^G}{\tilde{\beta_1^G}} + \frac{\epsilon_0}{2}.
\end{align*}
Meanwhile, the conditions $\gamma^G \in [0,1)$, (C3), and (A1)--(A2) guarantee that there exists $n_1 \in \mathbb{N}$ such that, for all $n \geq n_1$,
\begin{align*}
\gamma^{G^{n+1}} X_{n+1}^G(\hat{\bm{\theta}})
+
\mathrm{Dist}(\hat{\bm{\theta}})
\mathbb{E}\left[ 
\sum_{i\in [\Theta]} (h_{n+1,i}^G - h_{n,i}^G)
\right] \leq \frac{\alpha^G \tilde{\beta_1^G} \epsilon_0}{2}.
\end{align*}
Accordingly, from (\ref{ineq:0}), for all $n \geq n_2 := \max\{n_0,n_1\}$,
\begin{align*}
X_{n+1}^G(\hat{\bm{\theta}})
&< 
X_n^G(\hat{\bm{\theta}})
+
\frac{\alpha^G \tilde{\beta_1^G} \epsilon_0}{2}
+ \frac{\alpha^{G^2}}{(1-{\gamma^G})^2 h_{0,*}^G} \left( \frac{\sigma_G^2}{b} + M_G^2 \right)\\
&\quad + 2 \alpha^{G} 
\beta_1^G 
\sqrt{\Theta \mathrm{Dist}(\hat{\bm{\theta}})
\left( \frac{\sigma_G^2}{b} + M_G^2 \right)}
- 
2 \alpha^{G} \tilde{\beta_1^G}
\Bigg\{
\frac{\alpha^{G}}{2\tilde{\beta_1^G} (1-{\gamma^G})^2 h_{0,*}^G} \left( \frac{\sigma_G^2}{b} + M_G^2 \right)\\
&\quad 
+ \sqrt{\Theta \mathrm{Dist}(\hat{\bm{\theta}})
\left( \frac{\sigma_G^2}{b} + M_G^2 \right)}
\frac{\beta_1^G}{\tilde{\beta_1^G}} + \frac{\epsilon_0}{2}
\Bigg\}\\
&= 
X_n^G(\hat{\bm{\theta}}) - \frac{\alpha^G \tilde{\beta_1^G} \epsilon_0}{2}\\
&< 
X_{n_2}^G(\hat{\bm{\theta}}) - \frac{\alpha^G \tilde{\beta_1^G} \epsilon_0}{2}
(n+1 - n_2).
\end{align*}
Note that the right-hand side of the above inequality approaches minus infinity as $n$ approaches positive infinity, producing a contradiction. Therefore, (\ref{liminf_1}) holds, which implies that, for all $\bm{\theta} \in \mathbb{R}^\Theta$, 
\begin{align}\label{liminf_1_1}
\begin{split}
&\liminf_{n \to + \infty}
\mathbb{E} \left[ \langle \bm{\theta}_n - \bm{\theta}, \nabla_{\bm{\theta}}
L_G (\bm{\theta}_n, \bm{w}_n) \rangle \right]\\
&\quad \leq
\frac{\alpha^{G}}{2\tilde{\beta_1^G} (1-{\gamma^G})^2 h_{0,*}^G} \left( \frac{\sigma_G^2}{b} + M_G^2 \right)
+ \sqrt{\Theta \mathrm{Dist}(\bm{\theta})
\left( \frac{\sigma_G^2}{b} + M_G^2 \right)}
\frac{\beta_1^G}{\tilde{\beta_1^G}}.
\end{split}
\end{align}
A discussion similar to the one showing (\ref{liminf_1_1}), together with Lemma \ref{lem:3_2}, leads to the finding that, for all $\bm{w} \in \mathbb{R}^W$, 
\begin{align}\label{liminf_2}
\begin{split}
&\liminf_{n \to + \infty}
\mathbb{E} \left[ \langle \bm{w}_n - \bm{w}, \nabla_{\bm{w}}
L_D (\bm{\theta}_n, \bm{w}_n) \rangle \right]\\
&\quad \leq
\frac{\alpha^{D}}{2\tilde{\beta_1^D} (1-{\gamma^D})^2 h_{0,*}^D} \left( \frac{\sigma_D^2}{b} + M_D^2 \right)
+ \sqrt{W \mathrm{Dist}(\bm{w})
\left( \frac{\sigma_D^2}{b} + M_D^2 \right)}
\frac{\beta_1^D}{\tilde{\beta_1^D}}.
\end{split}
\end{align}
Let $\bm{\theta} \in \mathbb{R}^\Theta$. From (\ref{liminf_1_1}), there exists a subsequence $((\bm{\theta}_{n_i}, \bm{w}_{n_i}))_{i\in\mathbb{N}}$ of $((\bm{\theta}_n, \bm{w}_n))_{n\in\mathbb{N}}$ such that 
\begin{align*}
&\lim_{i \to + \infty}
\mathbb{E} \left[ \langle \bm{\theta}_{n_i} - \bm{\theta}, \nabla_{\bm{\theta}}
L_G (\bm{\theta}_{n_i}, \bm{w}_{n_i}) \rangle \right]\\
&\quad \leq
\frac{\alpha^{G}}{2\tilde{\beta_1^G} (1-{\gamma^G})^2 h_{0,*}^G} \left( \frac{\sigma_G^2}{b} + M_G^2 \right)
+ \sqrt{\Theta \mathrm{Dist}(\bm{\theta})
\left( \frac{\sigma_G^2}{b} + M_G^2 \right)}
\frac{\beta_1^G}{\tilde{\beta_1^G}}.
\end{align*}
Conditions (S1) and (C3) guarantee that there exists $((\bm{\theta}_{n_{i_j}}, \bm{w}_{n_{i_j}}))_{j\in\mathbb{N}}$ of $((\bm{\theta}_{n_i}, \bm{w}_{n_i}))_{i\in\mathbb{N}}$ such that $((\bm{\theta}_{n_{i_j}}, \bm{w}_{n_{i_j}}))_{j\in\mathbb{N}}$ converges almost surely to $(\bm{\theta}^*, \bm{w}^*) \in \mathbb{R}^\Theta \times \mathbb{R}^W$. Therefore, for all $\bm{\theta} \in \mathbb{R}^\Theta$,
\begin{align*}
&\mathbb{E} \left[ \langle \bm{\theta}^* - \bm{\theta}, \nabla_{\bm{\theta}}
L_G (\bm{\theta}^*, \bm{w}^*) \rangle \right]\\
&\quad \leq
\frac{\alpha^{G}}{2\tilde{\beta_1^G} (1-{\gamma^G})^2 h_{0,*}^G} \left( \frac{\sigma_G^2}{b} + M_G^2 \right)
+ \sqrt{\Theta \mathrm{Dist}(\bm{\theta})
\left( \frac{\sigma_G^2}{b} + M_G^2 \right)}
\frac{\beta_1^G}{\tilde{\beta_1^G}}.
\end{align*}
Hence, letting $\bm{\theta} = \tilde{\bm{\theta}} := \bm{\theta}^* - \nabla_{\bm{\theta}} L_G (\bm{\theta}^*, \bm{w}^*)$ implies that
\begin{align}\label{eq_1}
\begin{split}
&\mathbb{E} \left[ \| \nabla_{\bm{\theta}}
L_G (\bm{\theta}^*, \bm{w}^*) \|^2 \right]\\
&\quad \leq
\frac{\alpha^{G}}{2\tilde{\beta_1^G} (1-{\gamma^G})^2 h_{0,*}^G} \left( \frac{\sigma_G^2}{b} + M_G^2 \right)
+ \sqrt{\Theta \mathrm{Dist}(\tilde{\bm{\theta}})
\left( \frac{\sigma_G^2}{b} + M_G^2 \right)}
\frac{\beta_1^G}{\tilde{\beta_1^G}}.
\end{split}
\end{align}
A discussion similar to the one showing (\ref{eq_1}), together with (\ref{liminf_2}), implies that there exists $(\bm{\theta}_*, \bm{w}_*) \in \mathbb{R}^\Theta \times \mathbb{R}^W$ such that
\begin{align*}
&\mathbb{E} \left[ \| \nabla_{\bm{w}}
L_D (\bm{\theta}_*, \bm{w}_*) \|^2 \right]\\
&\quad \leq
\frac{\alpha^{D}}{2\tilde{\beta_1^D} (1-{\gamma^D})^2 h_{0,*}^D} \left( \frac{\sigma_D^2}{b} + M_D^2 \right)
+ \sqrt{W \mathrm{Dist}(\tilde{\bm{w}})
\left( \frac{\sigma_D^2}{b} + M_D^2 \right)}
\frac{\beta_1^D}{\tilde{\beta_1^D}},
\end{align*}
where $\tilde{\bm{w}} = \bm{w}_* - \nabla_{\bm{w}} L_D (\bm{\theta}_*, \bm{w}_*)$. This completes the proof. 
\end{proof}

\subsection{Proof of Theorem \ref{thm:1}(ii)}
\label{a6}
Lemmas \ref{lem:1_1} and \ref{lem:2_1} lead to the following lemma:

\begin{lem}\label{lem:4_1}
Suppose that (S1)--(S3), (A1)--(A2), and (C2)--(C3) hold and consider Algorithm \ref{algo:1}, where $(\alpha_n^G)_{n\in\mathbb{N}}$ is monotone decreasing. Then, for all $\bm{\theta} \in \mathbb{R}^\Theta$ and all $N \geq 1$,
\begin{align*}
&\frac{1}{N} \sum_{n\in [N]} \mathbb{E} \left[ \langle \bm{\theta}_n - \bm{\theta}, \nabla_{\bm{\theta}}
L_G (\bm{\theta}_n, \bm{w}_n) \rangle \right]\\
&\quad \leq
\frac{\Theta \mathrm{Dist}(\bm{\theta}) H^G}{2 \alpha_N^G \tilde{\beta_1^G} N}
+
\frac{1}{N}\sum_{n\in [N]}
\frac{\alpha_n^G}{2 \tilde{\beta_1^G}(1-{\gamma}^G)^2 h_{0,*}^G} 
\left( \frac{\sigma_G^2}{b} + M_G^2 \right)
+
\frac{\beta_1^G}{\tilde{\beta_1^G}}
\sqrt{\Theta \mathrm{Dist}(\bm{\theta})
\left( \frac{\sigma_G^2}{b} + M_G^2 \right)},
\end{align*}
where $H^G := \max_{i\in [\Theta]} H_i^G$.
\end{lem}

\begin{proof}
Lemma \ref{lem:1_1} implies that, for all $\bm{\theta}\in \mathbb{R}^\Theta$ and all $n\in\mathbb{N}$,
\begin{align*}
&\mathbb{E} \left[ \langle \bm{\theta}_n - \bm{\theta}, \nabla_{\bm{\theta}}
L_G (\bm{\theta}_n, \bm{w}_n) \rangle \right]\\
&= 
\frac{\tilde{\gamma}_n^G}{2 \alpha_n^{G} \tilde{\beta_1^G}}
\left\{
\mathbb{E}\left[ \| \bm{\theta}_{n} - \bm{\theta} \|_{\mathsf{H}_n^G}^2 \right]
-
\mathbb{E}\left[ \| \bm{\theta}_{n+1} - \bm{\theta} \|_{\mathsf{H}_n^G}^2 \right]
\right\}
+ \frac{\alpha_n^{G} \tilde{\gamma}_n^G}{2 \tilde{\beta_1^G}} \mathbb{E} \left[ \|\bm{d}_n^G \|_{\mathsf{H}_n^G}^2 \right]
+ 
\frac{\beta_1^G}{\tilde{\beta_1^G}} 
\mathbb{E} \left[ \langle \bm{\theta} - \bm{\theta}_n, \bm{m}_{n-1}^G \rangle \right], 
\end{align*}
which implies that, for all $\bm{\theta}\in \mathbb{R}^\Theta$ and all $N \geq 1$,
\begin{align*}
&\frac{1}{N} \sum_{n \in [N]} \mathbb{E} \left[ \langle \bm{\theta}_n - \bm{\theta}, \nabla_{\bm{\theta}}
L_G (\bm{\theta}_n, \bm{w}_n) \rangle \right]\\
&= 
\frac{1}{N} \underbrace{\sum_{n \in [N]}
\frac{\tilde{\gamma}_n^G}{2 \alpha_n^{G} \tilde{\beta_1^G}}
\left\{
\mathbb{E}\left[ \| \bm{\theta}_{n} - \bm{\theta} \|_{\mathsf{H}_n^G}^2 \right]
-
\mathbb{E}\left[ \| \bm{\theta}_{n+1} - \bm{\theta} \|_{\mathsf{H}_n^G}^2 \right]
\right\}}_{\Theta_N}
+ 
\frac{1}{N} \underbrace{\sum_{n \in [N]} \frac{\alpha_n^{G} \tilde{\gamma}_n^G}{2 \tilde{\beta_1^G}} \mathbb{E} \left[ \|\bm{d}_n^G \|_{\mathsf{H}_n^G}^2 \right]}_{A_N}\\
&\quad + 
\frac{1}{N} \underbrace{\sum_{n \in [N]} \frac{\beta_1^G}{\tilde{\beta_1^G}} 
\mathbb{E} \left[ \langle \bm{\theta} - \bm{\theta}_n, \bm{m}_{n-1}^G \rangle \right]}_{B_N}. 
\end{align*}
Let $\delta_n^G := \tilde{\gamma}_n^G/(2 \alpha_n^G \tilde{\beta_1^G})$ for all $n\in\mathbb{N}$. Then, we have that 
\begin{align*}
\Theta_N 
&:= 
\delta_1^G \mathbb{E}\left[ \| \bm{\theta}_{1} - \bm{\theta} \|_{\mathsf{H}_1^G}^2 \right]
+
\underbrace{\sum_{n=2}^N \left\{ 
\delta_n^G \mathbb{E}\left[ \| \bm{\theta}_{n} - \bm{\theta} \|_{\mathsf{H}_n^G}^2 \right] 
- \delta_{n-1}^G \mathbb{E}\left[ \| \bm{\theta}_{n} - \bm{\theta} \|_{\mathsf{H}_{n-1}^G}^2 \right]
\right\}}_{\tilde{\Theta}_N}\\
&\quad - \delta_N^G \mathbb{E}\left[ \| \bm{\theta}_{N+1} - \bm{\theta} \|_{\mathsf{H}_{N}^G}^2 \right].
\end{align*}
Since $\overline{\mathsf{H}}_n^G \in \mathbb{S}_{++}^\Theta$ exists such that $\mathsf{H}_n^G = \overline{\mathsf{H}}_n^{G^2}$, we have $\|\bm{x}\|_{\mathsf{H}_n^G}^2 = \| \overline{\mathsf{H}}_n^G \bm{x} \|^2$ for all $\bm{x}\in\mathbb{R}^\Theta$. Accordingly, we have 
\begin{align*}
\tilde{\Theta}_N 
=
\mathbb{E} \left[ 
\sum_{n=2}^N 
\left\{
\delta_n^G \left\| \overline{\mathsf{H}}_n^G (\bm{\theta}_{n} - \bm{\theta}) \right\|^2
-
\delta_{n-1}^G \left\| \overline{\mathsf{H}}_{n-1}^G (\bm{\theta}_{n} - \bm{\theta}) \right\|^2
\right\}
\right].
\end{align*}
Hence, for all $N\geq 2$,
\begin{align}\label{DELTA}
\tilde{\Theta}_N 
= 
\mathbb{E} \left[ 
\sum_{n=2}^N
\sum_{i=1}^\Theta 
\left(
\delta_n^G h_{n,i}^G
-
\delta_{n-1}^G h_{n-1,i}^G
\right)
(\theta_{n,i} - \theta_i)^2
\right].
\end{align}
Since $(\alpha_n^G)_{n\in\mathbb{N}}$ is monotone decreasing, we have that $\delta_{n+1}^G \geq \delta_{n}^G$ ($n\in \mathbb{N}$). Hence, from (A1), we have that, for all $n \geq 1$ and all $i\in [\Theta]$,
\begin{align*}
\delta_n^G h_{n,i}^G - \delta_{n-1}^G h_{n-1,i}^G \geq 0.
\end{align*} 
Moreover, from (C3), $\max_{i \in [\Theta]} \sup_{n\in\mathbb{N}} (\theta_{n,i} - \theta_i)^2 \leq \mathrm{Dist}(\bm{\theta})$. Accordingly, for all $N \geq 2$,
\begin{align*}
\tilde{\Theta}_N
&\leq
\mathrm{Dist}(\bm{\theta})
\mathbb{E} \left[ 
\sum_{n=2}^N
\sum_{i=1}^\Theta 
\left(
\delta_n^G h_{n,i}^G
-
\delta_{n-1}^G h_{n-1,i}^G
\right)
\right]\\
&= 
\mathrm{Dist}(\bm{\theta})
\mathbb{E} \left[ 
\sum_{i=1}^\Theta
\left(
\delta_N^G h_{N,i}^G
-
\delta_1^G h_{1,i}^G
\right)
\right].
\end{align*}
Therefore, $\delta_1^G \mathbb{E} [\| \bm{\theta}_{1} - \bm{\theta}\|_{\mathsf{H}_1^G}^2] \leq \mathrm{Dist}(\bm{\theta}) \delta_1^G \mathbb{E} [ \sum_{i=1}^\Theta h_{1,i}^G]$, and (A2) imply that, for all $N\geq 1$,
\begin{align*}
\Theta_N
&\leq
\delta_1^G \mathrm{Dist}(\bm{\theta}) \mathbb{E} \left[ 
\sum_{i=1}^\Theta h_{1,i}^G \right]
+
\mathrm{Dist}(\bm{\theta})
\mathbb{E} \left[
\sum_{i=1}^\Theta 
\left(
\delta_N^G h_{N,i}^G
-
\delta_1^G h_{1,i}^G
\right)
\right]\\
&=
\delta_N^G \mathrm{Dist}(\bm{\theta})
\mathbb{E} \left[
\sum_{i=1}^\Theta 
h_{N,i}^G
\right]\\
&\leq 
\delta_N^G \mathrm{Dist}(\bm{\theta}) 
\sum_{i=1}^\Theta 
H_{i}^G\\
&\leq 
\delta_N^G \Theta \mathrm{Dist}(\bm{\theta}) H^G,
\end{align*}
where $H^G = \max_{i\in [\Theta]} H_i^G$. From $\delta_n^G := \tilde{\gamma}_n^G/(2 \alpha_n^G \tilde{\beta_1^G})$ and $\tilde{\gamma}_n^G = 1 - {\gamma}^{G^{n+1}} \leq 1$, we have 
\begin{align}\label{L} 
\Theta_N 
\leq 
\frac{\Theta \mathrm{Dist}(\bm{\theta}) H^G}{2 \alpha_N^G \tilde{\beta_1^G}}.
\end{align}
Lemma \ref{lem:2_1} implies that, for all $N\geq 1$,
\begin{align}\label{D}
A_N := 
\sum_{n \in [N]} \frac{\alpha_n^{G} \tilde{\gamma}_n^G}{2 \tilde{\beta_1^G}} \mathbb{E} \left[ \|\bm{d}_n^G \|_{\mathsf{H}_n^G}^2 \right]
\leq 
\sum_{n\in [N]}
\frac{\alpha_n^G}{2 \tilde{\beta_1^G}(1-{\gamma}^G)^2 h_{0,*}^G} 
\left( \frac{\sigma_G^2}{b} + M_G^2 \right).
\end{align}
From (\ref{CS}), we have 
\begin{align}\label{B}
\begin{split}
B_N 
&:= 
\sum_{n\in[N]} \frac{\beta_1^G}{\tilde{\beta_1^G}} \mathbb{E} \left[ \langle \bm{\theta} - \bm{\theta}_n, \bm{m}_{n-1}^G \rangle \right]
\leq
\frac{\beta_1^G N}{\tilde{\beta_1^G}}
\sqrt{\Theta \mathrm{Dist}(\bm{\theta})
\left( \frac{\sigma_G^2}{b} + M_G^2 \right)}.
\end{split}
\end{align}
Therefore, (\ref{L}), (\ref{D}), and (\ref{B}) lead to the assertion in Lemma \ref{lem:4_1}. This completes the proof.
\end{proof}

A discussion similar to the one proving Lemma \ref{lem:4_1}, together with Lemmas \ref{lem:1_2} and \ref{lem:2_2}, leads to the following lemma:

\begin{lem}\label{lem:4_2}
Suppose that (S1)--(S3), (A1)--(A2), and (C2)--(C3) hold and consider Algorithm \ref{algo:1}, where $(\alpha_n^D)_{n\in\mathbb{N}}$ is monotone decreasing. Then, for all $\bm{w} \in \mathbb{R}^W$ and all $N \geq 1$,
\begin{align*}
&\frac{1}{N} \sum_{n\in [N]} \mathbb{E} \left[ \langle \bm{w}_n - \bm{w}, \nabla_{\bm{w}}
L_D (\bm{\theta}_n, \bm{w}_n) \rangle \right]\\
&\quad \leq
\frac{W \mathrm{Dist}(\bm{w}) H^D}{2 \alpha_N^D \tilde{\beta_1^D} N}
+
\frac{1}{N}\sum_{n\in [N]}
\frac{\alpha_n^D}{2 \tilde{\beta_1^D}(1-{\gamma}^D)^2 h_{0,*}^D} 
\left( \frac{\sigma_D^2}{b} + M_D^2 \right)
+
\frac{\beta_1^D}{\tilde{\beta_1^D}}
\sqrt{W \mathrm{Dist}(\bm{w})
\left( \frac{\sigma_D^2}{b} + M_D^2 \right)},
\end{align*}
where $H^D := \max_{j\in [W]} H_j^D$.
\end{lem}

\begin{proof}
[Proof of Theorem \ref{thm:1}(ii)] Let $\alpha_n^G := \alpha^G$ and $\alpha_n^D := \alpha^D$ for all $n\in\mathbb{N}$. Lemmas \ref{lem:4_1} and \ref{lem:4_2} thus guarantee that, for all $\bm{\theta} \in \mathbb{R}^\Theta$ and all $N \geq 1$,
\begin{align*}
&\frac{1}{N} \sum_{n\in [N]} \mathbb{E} \left[ \langle \bm{\theta}_n - \bm{\theta}, \nabla_{\bm{\theta}}
L_G (\bm{\theta}_n, \bm{w}_n) \rangle \right]\\
&\quad \leq
\frac{\Theta \mathrm{Dist}(\bm{\theta}) H^G}{2 \alpha^G \tilde{\beta_1^G} N}
+
\frac{\alpha^G}{2 \tilde{\beta_1^G}(1-{\gamma}^G)^2 h_{0,*}^G} 
\left( \frac{\sigma_G^2}{b} + M_G^2 \right)
+
\frac{\beta_1^G}{\tilde{\beta_1^G}}
\sqrt{\Theta \mathrm{Dist}(\bm{\theta})
\left( \frac{\sigma_G^2}{b} + M_G^2 \right)},\\
&\frac{1}{N} \sum_{n\in [N]} \mathbb{E} \left[ \langle \bm{w}_n - \bm{w}, \nabla_{\bm{w}}
L_D (\bm{\theta}_n, \bm{w}_n) \rangle \right]\\
&\quad \leq
\frac{W \mathrm{Dist}(\bm{w}) H^D}{2 \alpha^D \tilde{\beta_1^D} N}
+
\frac{\alpha^D}{2 \tilde{\beta_1^D}(1-{\gamma}^D)^2 h_{0,*}^D} 
\left( \frac{\sigma_D^2}{b} + M_D^2 \right)
+
\frac{\beta_1^D}{\tilde{\beta_1^D}}
\sqrt{W \mathrm{Dist}(\bm{w})
\left( \frac{\sigma_D^2}{b} + M_D^2 \right)},
\end{align*}
which completes the proof.
\end{proof}

\subsection{Proof of Theorem \ref{thm:2}}
\label{a7}
\begin{proof}
[Proof of Theorem \ref{thm:2}] Condition (\ref{lower}) implies that TTUR achieves an $\epsilon$--approximation. We have that, for $b > B_G /(\epsilon_G^2 - C_G)$,
\begin{align*}
&\frac{\mathrm{d} N_G (b)}{\mathrm{d} b}
= 
\frac{- A_G B_G}{\{(\epsilon_G^2 - C_G)b - B_G\}^2} \leq 0,\\
&\frac{\mathrm{d}^2 N_G(b)}{\mathrm{d} b^2}
= 
\frac{2 A_G B_G(\epsilon_G^2 - C_G)}{\{(\epsilon_G^2 - C_G)b - B_G\}^3} \geq 0.
\end{align*}
Hence, $N_G$ is monotone decreasing and convex for $b > B_G /(\epsilon_G^2 - C_G)$. We also have that $N_D$ is monotone decreasing and convex for $b > B_D /(\epsilon_D^2 - C_D)$. This completes the proof.
\end{proof}

\subsection{Proof of Theorem \ref{thm:3}}
\label{a8}
\begin{proof}
[Proof of Theorem \ref{thm:3}] We have that, for $b > B_G/(\epsilon_G^2 - C_G)$, 
\begin{align*}
N_G(b) b = 
\frac{A_G b^2}
{(\epsilon_G^2 - C_G) b - B_G}.
\end{align*}
Hence, 
\begin{align*}
&\frac{\mathrm{d} N_G (b) b}{\mathrm{d} b}
= 
\frac{A_G b \{(\epsilon_G^2 - C_G)b - 2 B_G \}}
{\{(\epsilon_G^2 - C_G) b - B_G \}^2},\\
&\frac{\mathrm{d}^2 N_G (b) b}{\mathrm{d} b^2}
= 
\frac{2 A_G B_G^2}{\{(\epsilon_G^2 - C_G)b - B_G\}^3} \geq 0,
\end{align*}
which implies that $N_G (b) b$ is convex for $b > B_G/(\epsilon_G^2 - C_G)$ and 
\begin{align*}
\frac{\mathrm{d} N_G (b) b}{\mathrm{d} b}
\begin{cases}
< 0 &\text{ if } b < b_G^\star,\\
= 0 &\text{ if } b = b_G^\star = \frac{2B_G}{\epsilon_G^2 - C_G},\\
> 0 &\text{ if } b > b_G^\star.
\end{cases}
\end{align*}
The point $b_G^\star$ attains the minimum value $N_G(b_G^\star) b_G^\star$ of $N_G (b) b$. The discussion for $N_D$ is similar to the one for $N_G$. This completes the proof.
\end{proof}

\subsection{Proof of Proposition \ref{prop:4}}
\label{app:prop4}
Theorem \ref{thm:3} and the definitions of $B_G$, $C_G$, $\tilde{\beta_1^G}$ and $\tilde{\gamma}^G$ ensure that
\begin{align*}
b_G^\star 
= \frac{2 B_G}{\epsilon_G^2 - C_G}
> \frac{2 B_G}{\epsilon_G^2}
= \frac{2}{\epsilon_G^2} \cdot \frac{\sigma_G^2 \alpha^G}{2 \tilde{\beta_1^G} \tilde{\gamma}^{G^2} h_{0,*}^G}
= \frac{\sigma_G^2 \alpha^G}{\epsilon_G^2 (1-\beta_1^G)(1-\gamma^G)^2 h_{0,*}^G} .
\end{align*}
Moreover, (\ref{A3}) and the definition of $L_{G}(\bm{\theta}_{n}, \bm{w}_n)$ ensure that
\begin{align*}
\left\| \nabla L_{G,\mathcal{S}_{n}}(\bm{\theta}_{n}) \right\|^2 
\leq 
\frac{\sigma_G^2}{b} + \left\|\nabla L_{G}(\bm{\theta}_{n}, \bm{w}_n) \right\|^2
\approx 
\left\|\nabla L_{G}(\bm{\theta}_{n}, \bm{w}_n) \right\|^2
=\frac{1}{|S|^2} \left\| \sum_{i=1}^{|S|} \nabla L_{G}^{(i)}(\bm{\theta}_{n}, \bm{w}_n) \right\|^2, 
\end{align*}
where $\sigma_G^2/b \approx 0$ holds when $b$ is sufficiently large. We define $\sum_{i=1}^{|S|} \nabla L_{G}^{(i)}(\bm{\theta}_{n}, \bm{w}_n) := \bm{G}_{n} $ for all $n \in \mathbb{N}$. Then, we have
\begin{align*}
\frac{1}{|S|^2} \left\| \sum_{i=1}^{|S|} \nabla L_{G}^{(i)}(\bm{\theta}_{n}, \bm{w}_n) \right\|^2
=\frac{1}{|S|^2}\sum_{i=1}^{\Theta} G_{n,i}^2
\leq \frac{\Theta}{|S|^2}\max_{i \in [\Theta]} G_{n,i}^2 .
\end{align*}

\begin{proof}
[Proof of Proposition \ref{prop:4}(i)] The definition of $\bm{v}_n^G$ implies that
\begin{align*}
v_{n,i}^G 
= \beta_2^G v_{n-1,i}^G + (1- \beta_2^G) g_{n,i}^2
\leq \max_{n,i} g_{n,i}^2
=: g_{n^*,i^*}^2
\leq \sum_{i=1}^{\Theta} g_{n^*,i}^2
= \left\| \nabla L_{G,\mathcal{S}_{n^*}}(\bm{\theta}_{n^*}) \right\|^2 ,
\end{align*}
where the first inequality can be shown by induction. Therefore, for all $n \in \mathbb{N}$ and all $i \in [\Theta]$,
\begin{align*}
v_{n,i}^G 
\leq \frac{\Theta}{|S|^2} \max_{i \in [\Theta]} G_{n^*,i}^2 .
\end{align*}
From the definition of $\bar{v}_{n,i}^G$ and $\beta_2 \in [0,1)$, we have
\begin{align*}
h_{0,*}^G
:=\min_{i \in [\Theta]} h_{0,i}^G
\leq \sqrt{\bar{v}_{n,i}^G}
= \sqrt{\frac{v_{n,i}^G}{1-{\beta_2^G}^n}}
\leq \sqrt{\frac{v_{n,i}^G}{1-\beta_2^G}}.
\end{align*}
Hence,
\begin{align*}
b_G^\star \geq \frac{\sigma_G^2 \alpha^G}{\epsilon_G^2 (1-\beta_1^G)(1-\gamma^G)^2 \sqrt{\frac{\Theta}{1-\beta_2^G} \frac{1}{|S|^2} \max_{i \in [\Theta]} G_{n^*,i}^2}} .
\end{align*}
Accordingly, using $\gamma^G = \beta_1^G$ and $\max_{i \in [\Theta]} G_{n^*,i}^2 \approx \epsilon_G^2$ implies that
\begin{align*}
b_G^\star \geq \frac{\sigma_G^2 \alpha^G}{\epsilon_G^3 (1-\beta_1^G)^3 \sqrt{\frac{\Theta}{1-\beta_2^G} \frac{1}{|S|^2}}} .
\end{align*}
The discussion for $b_D^\star$ is similar to the one for $b_G^\star$. This completes the proof.
\end{proof}

\begin{proof}
[Proof of Proposition\ref{prop:4}(ii)] The definition of $\bm{s}_n^G$ implies that
\begin{align*}
s_{n,i}^G 
&= \beta_2^G v_{n-1,i}^G + (1- \beta_2^G) (g_{n,i} - m_{n,i})^2
\leq \max_{n,i} (g_{n,i} - m_{n,i})^2
=: (g_{n^*,i^*} - m_{n^*,i^*})^2 \\
&\leq \sum_{i=1}^{\Theta} (g_{n^*,i} - m_{n^*,i})^2
= \left\| \nabla L_{G,\mathcal{S}_{n^*}}(\bm{\theta}_{n^*}) - \bm{m}_{n^*}^G \right\|^2 ,
\end{align*}
where the first inequality can be shown by induction. Hence, we have
\begin{align*}
\left\| \nabla L_{G,\mathcal{S}_{n^*}}(\bm{\theta}_{n^*}) - \bm{m}_{n^*}^G \right\|^2
\leq
2 \left\| \nabla L_{G,\mathcal{S}_{n^*}}(\bm{\theta}_{n^*}) \right\|^2 + 2 \left\| \bm{m}_{n^*}^G \right\|^2
\leq
\frac{4 \sigma_G^2}{b} + 4 \left\|\nabla L_{G}(\bm{\theta}_{n^*}, \bm{w}_{n^*}) \right\|^2
\end{align*}
Therefore, for all $n \in \mathbb{N}$ and all $i \in [\Theta]$,
\begin{align*}
s_{n,i}^G 
\leq \frac{4\Theta}{|S|^2}\max_{i \in [\Theta]} G_{n^*,i}^2 .
\end{align*}
From the definition of $\hat{s}_{n,i}^G$ and $\beta_2 \in [0,1)$, we have
\begin{align*}
h_{0,*}^G
:=\min_{i \in [\Theta]} h_{0,i}^G
\leq \sqrt{\hat{s}_{n,i}^G}
= \sqrt{\frac{s_{n,i}^G}{1-{\beta_2^G}^n}}
\leq \sqrt{\frac{s_{n,i}^G}{1-\beta_2^G}}.
\end{align*}
Hence,
\begin{align*}
b_G^\star \geq \frac{\sigma_G^2 \alpha^G}{\epsilon_G^2 (1-\beta_1^G)(1-\gamma^G)^2 \sqrt{\frac{4\Theta}{1-\beta_2^G} \frac{1}{|S|^2} \max_{i \in [\Theta]} G_{n^*,i}^2}} .
\end{align*}
Accordingly, using $\gamma^G = \beta_1^G$ and $\max_{i \in [\Theta]} G_{n^*,i}^2 \approx \epsilon_G^2$ implies that
\begin{align*}
b_G^\star \geq \frac{\sigma_G^2 \alpha^G}{\epsilon_G^3 (1-\beta_1^G)^3 \sqrt{\frac{4\Theta}{1-\beta_2^G} \frac{1}{|S|^2}}} .
\end{align*}
The discussion for $b_D^\star$ is similar to the one for $b_G^\star$. This completes the proof.
\end{proof}

\begin{proof}
[Proof of Proposition \ref{prop:4}(iii)] From the definition of $\bm{v}_n^G$ and a discussion similar to the proof of Proposition \ref{prop:4}(i), we have
\begin{align*}
v_{n,i}^G
\leq
\frac{\Theta}{|S|^2} \max_{i \in [\Theta]} G_{n^*,i}^2 .
\end{align*}
From the definition of $v_{n,i}^G$, we have
\begin{align*}
h_{0,*}^G
:=\min_{i \in [\Theta]} h_{0,i}^G
\leq
\sqrt{v_{n,i}^G}.
\end{align*}
Hence,
\begin{align*}
b_G^\star \geq \frac{\sigma_G^2 \alpha^G}{\epsilon_G^2 (1-\beta_1^G)(1-\gamma^G)^2 \sqrt{\frac{\Theta}{|S|^2} \max_{i \in [\Theta]} G_{n^*,i}^2}} .
\end{align*}
Accordingly, using $\gamma^G = \beta_1^G = 0$ and $\max_{i \in [\Theta]} G_{n^*,i}^2 \approx \epsilon_G^2$ implies that
\begin{align*}
b_G^\star \geq \frac{\sigma_G^2 \alpha^G}{\epsilon_G^3 \sqrt{\frac{\Theta}{|S|^2}}} .
\end{align*}
The discussion for $b_D^\star$ is similar to the one for $b_G^\star$. This completes the proof.
\end{proof}

\subsection{Relationship between stationary point problem and variational inequality}
\label{a10}
\begin{prop}\label{prop:1}
Suppose that $f: \mathbb{R}^d \to \mathbb{R}$ is continuously differentiable and $\bm{x}^*$ is a stationary point of $f$. Then, $\nabla f(\bm{x}^*) = \bm{0}$ is equivalent to the following variational inequality: for all $\bm{x} \in \mathbb{R}^d$,
\begin{align*}
\langle \nabla f(\bm{x}^*), \bm{x}-\bm{x}^* \rangle \geq 0.
\end{align*}

\begin{proof}
[Proof of Proposition \ref{prop:1}] Suppose that $\bm{x} \in \mathbb{R}^d$ satisfies $\nabla f(\bm{x}) = \bm{0}$. Then, for all $\bm{y} \in \mathbb{R}^d$,
\begin{align*}
\langle \nabla f(\bm{x}), \bm{y}-\bm{x} \rangle 
\geq
0.
\end{align*}
Suppose that $\bm{x} \in \mathbb{R}^d$ satisfies $\langle \nabla f(\bm{x}), \bm{y}-\bm{x} \rangle \geq 0$ for all $\bm{y} \in \mathbb{R}^d$. Let $\bm{y} := \bm{x} - \nabla f(\bm{x})$. Then we have
\begin{align*}
0 \leq
\langle \nabla f(\bm{x}), \bm{y}-\bm{x} \rangle
= -\| \nabla f(\bm{x}) \|^2.
\end{align*} 
Hence,
\begin{align*}
\nabla f(\bm{x}) = \bm{0}.
\end{align*}
\end{proof}
\end{prop}

\subsection{Remarks regarding (C3)}
\label{c3}
We make the following remarks regarding (C3). 

(C3)(i)
Let $L_G^{(i)}(\cdot, \bm{w}) \colon \mathbb{R}^\Theta \to \mathbb{R}$ $(i\in \mathcal{S})$ be convex and $\nabla_{\bm{\theta}} L_G^{(i)}(\cdot, \bm{w}) \colon \mathbb{R}^\Theta \to \mathbb{R}^\Theta$ be Lipschitz continuous with the Lipschitz constant $l_G$. Let $L_D^{(i)}(\bm{\theta},\cdot) \colon \mathbb{R}^W \to \mathbb{R}$ $(i\in \mathcal{R})$ be convex and $\nabla_{\bm{w}} L_D^{(i)}(\bm{\theta},\cdot) \colon \mathbb{R}^W \to \mathbb{R}^W$ be Lipschitz continuous with the Lipschitz constant $l_D$. Then, the sequences $(\bm{\theta}_n)_{n\in\mathbb{N}}$ and $(\bm{w}_n)_{n\in\mathbb{N}}$ generated by SGD with $\alpha^G \leq 1/l_G$ and $\alpha_D \leq 1/l_D$ satisfy (C3).

(C3)(ii)
Let $L_G^{(i)}(\cdot, \bm{w}) \colon \mathbb{R}^\Theta \to \mathbb{R}$ $(i\in \mathcal{S})$ and $L_D^{(i)}(\bm{\theta},\cdot) \colon \mathbb{R}^W \to \mathbb{R}$ $(i\in \mathcal{R})$ be nonconvex. Let $B^G \subset \mathbb{R}^\Theta$ and $B^D \subset \mathbb{R}^W$ be closed balls defined by $B^G := \{ \bm{\theta} \in \mathbb{R}^\Theta \colon \| \bm{\theta} - \bm{c}^G \| \leq r^G \}$ and $B^D := \{ \bm{w} \in \mathbb{R}^W \colon \| \bm{w} - \bm{c}^D \| \leq r^D \}$, where $\bm{c}^G \in \mathbb{R}^\Theta$, $\bm{c}^D \in \mathbb{R}^W$, and $r^G, r^D > 0$ are large radii of $B^G$ and $B^D$. We replace $\bm{\theta}_{n+1} := \bm{\theta}_n + \alpha_n^G \bm{d}^{G}_n$ and $\bm{w}_{n+1} := \bm{w}_n + \alpha_n^D \bm{d}^{D}_n$ in Algorithm \ref{algo:1} with 
\begin{align}\label{c_b}
\bm{\theta}_{n+1} := P_{G} \left(\bm{\theta}_n + \alpha_n^G \bm{d}^{G}_n \right) \text{ and }
\bm{w}_{n+1} := P_{D} \left( \bm{w}_n + \alpha_n^D \bm{d}^{D}_n \right),
\end{align}
where $P_{G}$ is the projection onto $B^G$ and $P_{D}$ is the projection onto $B^D$. Then, the sequences $(\bm{\theta}_n)_{n\in\mathbb{N}} \subset B^G$ and $(\bm{w}_n)_{n\in\mathbb{N}} \subset B^D$ generated by (\ref{c_b}) are bounded. Thus, the nonexpansivity conditions of $P_{G}$ and $P_{D}$ guarantee that Algorithm \ref{algo:1} using (\ref{c_b}) satisfies (\ref{main}) without assuming (C3).

\begin{proof}
[Proof of (C3)(i)] Let $n\in \mathbb{N}$. We define $\phi_n \colon \mathbb{R}^\Theta \to \mathbb{R}$ for all $\bm{\theta} \in \mathbb{R}^\Theta$ by 
\begin{align*}
 \phi_n (\bm{\theta}) 
 := L_{G, \mathcal{S}_n}(\bm{\theta}_n,\bm{w}_n) + \langle \nabla L_{G, \mathcal{S}_n} (\bm{\theta}_n), \bm{\theta}-\bm{\theta}_n \rangle + \frac{1}{2\alpha^G} \|\bm{\theta} - \bm{\theta}_n\|^2.
\end{align*}
The convexity of $L_{G, \mathcal{S}_n}(\cdot,\bm{w}_n)$ implies that, for all $\bm{\theta} \in \mathbb{R}^\Theta$, $L_{G, \mathcal{S}_n}(\bm{\theta},\bm{w}_n) \geq L_{G, \mathcal{S}_n}(\bm{\theta}_n,\bm{w}_n) + \langle \nabla L_{G, \mathcal{S}_n} (\bm{\theta}_n), \bm{\theta}-\bm{\theta}_n \rangle$. Hence, we have that, for all $\bm{\theta} \in \mathbb{R}^\Theta$, 
\begin{align}\label{ineq:sdm_convex_0}
 \phi_n (\bm{\theta})
 \leq 
 L_{G, \mathcal{S}_n}(\bm{\theta},\bm{w}_n)
 + \frac{1}{2\alpha^G} \|\bm{\theta} - \bm{\theta}_n\|^2.
\end{align}
Moreover, from $\nabla_{\bm{\theta}} \phi_n (\bm{\theta}) = \nabla L_{G, \mathcal{S}_n} (\bm{\theta}_n) + (1/\alpha^G)(\bm{\theta} - \bm{\theta}_n)$, $\phi_n$ is strongly convex with a constant $1/\alpha^G$. Accordingly, there exists a unique minimizer $\tilde{\bm{\theta}}_n$ of $\phi_n$ such that 
\begin{align*}
 \bm{0}
 = \nabla_{\bm{\theta}} \phi_n (\tilde{\bm{\theta}}_n) 
 = \nabla L_{G, \mathcal{S}_n} (\bm{\theta}_n)
+ \frac{1}{\alpha^G} (\bm{\theta} - \bm{\theta}_n),
\end{align*}
i.e.,
\begin{align*}
 \tilde{\bm{\theta}}_n = \bm{\theta}_n - \alpha^G \nabla L_{G, \mathcal{S}_n} (\bm{\theta}_n) =: \bm{\theta}_{n+1}.
\end{align*}
The strong convexity of $\phi_n$ and $\nabla_{\bm{\theta}} \phi_n (\bm{\theta}_{n+1}) = \bm{0}$ imply that 
\begin{align}\label{ineq:sdm_convex_1}
\begin{split}
 \phi_n (\bm{\theta})
 &\geq 
 \phi_n (\bm{\theta}_{n+1})
 + 
 \langle \nabla_{\bm{\theta}} \phi_n (\bm{\theta}_{n+1}), \bm{\theta}-\bm{\theta}_{n+1} \rangle 
 +
 \frac{1}{2 \alpha^G}\|\bm{\theta}_{n+1} - \bm{\theta}\|^2\\
 &= 
 \phi_n (\bm{\theta}_{n+1})
 +
 \frac{1}{2 \alpha^G}\|\bm{\theta}_{n+1} - \bm{\theta}\|^2.
\end{split}
\end{align}
The condition $\alpha^G \leq 1/l_G$ and the Lipschitz continuity of $\nabla L_{G, \mathcal{S}_n} (\cdot)$ ensure that
\begin{align}\label{ineq:sdm_convex_2}
\begin{split}
 \phi_n(\bm{\theta}_{n+1})
 &:= 
 L_{G, \mathcal{S}_n}(\bm{\theta}_n,\bm{w}_n) + \langle \nabla L_{G, \mathcal{S}_n} (\bm{\theta}_n), \bm{\theta}_{n+1}-\bm{\theta}_n \rangle + \frac{1}{2\alpha^G} \|\bm{\theta}_{n+1} - \bm{\theta}_n\|^2\\
 &\geq 
 L_{G, \mathcal{S}_n}(\bm{\theta}_n,\bm{w}_n) + \langle \nabla L_{G, \mathcal{S}_n} (\bm{\theta}_n), \bm{\theta}_{n+1}-\bm{\theta}_n \rangle + \frac{l_G}{2} \|\bm{\theta}_{n+1} - \bm{\theta}_n\|^2\\
 &\geq 
 L_{G, \mathcal{S}_n}(\bm{\theta}_{n+1},\bm{w}_n).
\end{split} 
\end{align}
From (\ref{ineq:sdm_convex_0}), (\ref{ineq:sdm_convex_1}), and (\ref{ineq:sdm_convex_2}), 
\begin{align*}
 L_{G, \mathcal{S}_n}(\bm{\theta},\bm{w}_n) + \frac{1}{2 \alpha^G}\|\bm{\theta}_n -\bm{\theta}\|^2
 \geq
 L_{G, \mathcal{S}_n}(\bm{\theta}_{n+1},\bm{w}_n) + \frac{1}{2 \alpha^G}\|\bm{\theta}_{n+1} -\bm{\theta}\|^2,
\end{align*}
which implies that, for all $\bm{\theta} \in \mathbb{R}^\Theta$,
\begin{align*}
 L_{G}(\bm{\theta},\bm{w}_n) + \frac{1}{2 \alpha^G}\|\bm{\theta}_n -\bm{\theta}\|^2
 \geq
 L_{G}(\bm{\theta}_{n+1},\bm{w}_n) + \frac{1}{2 \alpha^G}\|\bm{\theta}_{n+1} -\bm{\theta}\|^2.
\end{align*}
Let $\bm{\theta}^* \in \mathbb{R}^\Theta$ be a minimizer of $L_{G}(\bm{\theta},\bm{w}_n)$. Then, we have 
\begin{align*}
\frac{1}{2 \alpha^G} \left( 
\|\bm{\theta}_{n+1} -\bm{\theta}^*\|^2 - 
\|\bm{\theta}_{n} -\bm{\theta}^*\|^2
\right)
\leq 
L_{G}(\bm{\theta}^*,\bm{w}_n) - L_{G}(\bm{\theta}_{n+1},\bm{w}_n)
\leq 0.
\end{align*}
Since $(\|\bm{\theta}_{n} -\bm{\theta}^*\|^2)_{n\in \mathbb{N}}$ is monotone decreasing, the sequence $(\|\bm{\theta}_{n} -\bm{\theta}^*\|^2)_{n\in \mathbb{N}}$ is bounded. The discriminator can be defined by replacing $G$, $\mathcal{S}_n$, $\bm{\theta}$, and $\bm{w}$ in the generator with $D$, $\mathcal{R}_n$, $\bm{w}$, and $\bm{\theta}$.
\end{proof}

Details of (C3)(ii): Let $\bm{\theta} \in B^G \subset \mathbb{R}^\Theta$ and $n\in\mathbb{N}$. The definition of $\bm{\theta}_{n+1}$ and the nonexpansivity of $P_G$ (i.e., $\|P_G(\bm{\theta}_1) - P_G(\bm{\theta}_2)\|_{\mathsf{H}} \leq \|\bm{\theta}_1 - \bm{\theta}_2\|_{\mathsf{H}}$ $(\bm{\theta}_1,\bm{\theta}_2 \in \mathbb{R}^\Theta, \mathsf{H} \in \mathbb{S}_{++}^\Theta)$) imply that 
\begin{align*}
\| \bm{\theta}_{n+1} - \bm{\theta} \|_{\mathsf{H}_n^G}^2
&=
\| P_G(\bm{\theta}_{n} + \alpha_n^G \bm{d}_n^G) - P_G(\bm{\theta}) \|_{\mathsf{H}_n^G}^2\\
&\leq 
\| (\bm{\theta}_{n} - \bm{\theta} )+ \alpha_n^G \bm{d}_n^G \|_{\mathsf{H}_n^G}^2
=
\| \bm{\theta}_{n} - \bm{\theta} \|_{\mathsf{H}_n^G}^2
+ 2 \alpha_n^G \langle \bm{\theta}_{n} - \bm{\theta}, \bm{d}_n^G \rangle_{\mathsf{H}_n^G}
+ \alpha_n^{G^2} \|\bm{d}_n^G \|_{\mathsf{H}_n^G}^2.
\end{align*}
Hence, a discussion similar to the one proving Theorem \ref{thm:1} ensures that Algorithm \ref{algo:1} using (\ref{c_b}) satisfies (\ref{main}) without assuming (C3).
\end{document}